\documentclass{article}

\PassOptionsToPackage{numbers, compress}{natbib}
\usepackage[preprint]{neurips_2026}

\usepackage[utf8]{inputenc}
\usepackage[T1]{fontenc}
\usepackage{hyperref}
\usepackage{url}
\usepackage{booktabs}
\usepackage{amsfonts}
\usepackage{nicefrac}
\usepackage{microtype}
\usepackage{xcolor}
\usepackage{graphicx}
\usepackage{subcaption}
\usepackage{algorithm}
\usepackage{algorithmicx}
\usepackage{algpseudocode}
\usepackage{multirow}
\usepackage{xspace}
\usepackage{enumitem}
\usepackage{amsmath}
\usepackage{tabularx}
\usepackage{makecell}
\usepackage{adjustbox}
\usepackage{tikz}
\usepackage{svg}
\usepackage{listings}
\usepackage{fontawesome}
\usepackage{amsthm}
\usepackage{amssymb}
\usepackage{commands}
\usepackage{thmtools}
\usepackage[labelfont=bf,textfont=normal]{caption}
\usepackage{array}
\usepackage{wrapfig}
\usepackage{float}
\usepackage{tcolorbox}
\usepackage[capitalize,noabbrev]{cleveref}
\title{HalluTracer: Pre-Decoding Truthfulness Prediction via Depth-Averaged Probe-Logit}

\author{%
  \normalfont
  \textbf{Zhihao Guo}\textsuperscript{1,2}\thanks{Equal contribution.} \quad
  \textbf{Zonghan Wu}\textsuperscript{1}\footnotemark[1] \quad
  \textbf{Huan Huo}\textsuperscript{2} \quad
  \textbf{DaYong Ye}\textsuperscript{3} \\
  \normalfont
  \textbf{Junwei Zhang}\textsuperscript{4} \quad
  \textbf{Weiran Yao}\textsuperscript{5} \quad
  \textbf{Zhiwei Liu}\textsuperscript{6} \quad
  \textbf{Qingsong Wen}\textsuperscript{1,7}\thanks{Correspondence: qingsongedu@gmail.com, yileishao@sem.ecnu.edu.cn} \quad
  \textbf{Yilei Shao}\textsuperscript{1}\footnotemark[2]
  \\[0.6em]
  \normalfont 
  \textsuperscript{1}East China Normal University, China \quad
  \textsuperscript{2}University of Technology Sydney, Australia \\
  \textsuperscript{3}City University of Macau, Macau \quad
  \textsuperscript{4}Meta, USA \quad 
  \textsuperscript{5}actAVA AI, USA \\
  \textsuperscript{6}Microsoft AI, USA
  \textsuperscript{7}Squirrel AI Learning, USA
  \\[0.4em]
  \normalfont 
}

\begin{document}

\maketitle

\begin{abstract}
Internal-state probes enable truthfulness prediction before a large language model generates an answer.
When detectors change both the layers they read and the rules used to combine them, the source of improved prediction becomes difficult to identify.
We separate these choices and find that retaining more layers improves prediction even under fixed equal weighting.
An exact Fisher-ratio decomposition explains why the additional benefit of linear reweighting is limited on these probe scores: information in layer-wise differences largely overlaps with that captured by the depth mean.
Estimating additional weights can then offset this small benefit when the data used to fit them are limited.
These findings motivate our proposed method \thistool{}, which averages layer-wise probe logits to predict truthfulness before decoding.
Across six models and four benchmarks, including TruthfulQA, \thistool{} achieves the highest AUROC in 23 of 24 model--benchmark pairs among the compared methods.
The results support using evidence from across the network without requiring a correspondingly more flexible aggregation rule, clarifying the distinct roles of layer selection and weighting in pre-decoding detection.
\end{abstract}
\section{Introduction}
\label{sec:intro}

Large language models can produce fluent text that contains factual errors~\citep{ji2023survey,zhang2023siren,huang2025survey}.  Such errors need not be caught only after the text is produced, because truth-relevant evidence is already present before generation: linear probes can recover truthfulness directions from hidden activations~\citep{marks2023geometry}, and prompted models can report calibrated beliefs about their own answers~\citep{kadavath2022language}.  Pre-decoding detectors build on this observation and read the evidence before any token is emitted~\citep{ji2024query,alnuhait2025factcheckmate}. This evidence, however, is spread over many internal readouts, so a detector must decide which readouts to retain and how to combine them into one scalar score.

Existing white-box studies use selected layers, attention heads, or hidden-state representations~\citep{ccs,iti,fact-probe}, while other work combines information across depth through layer contrasts and learned modules~\citep{iris,dola,chen2024inside,lsd2025,beigi2024internalinspector,barshalom2025actvit}. Recent uncertainty and deception probes also average layer-wise scores across depth~\citep{xiao2026eagle,taufeeque2026obfuscation}. These approaches motivate understanding how much learning aggregation weights can improve on a simple average. For a fixed set of layer-wise scores, Fisher's linear discriminant provides a reference by maximising the pooled signal-to-noise ratio~\citep{fisher1936use,su1993linear,fan2012road}. Retaining more layers increases the number of cross-layer covariances that Fisher weighting must estimate from the same labelled examples. With limited training data, estimation error in this covariance matrix can undermine the additional benefit of learned weighting~\citep{bickel2004some,ledoit2004well}. The remaining question is whether the additional gain over equal weighting outweighs this estimation cost in pre-decoding truthfulness detection.

\begin{figure}[t]
\centering
\includegraphics[width=0.95\linewidth]{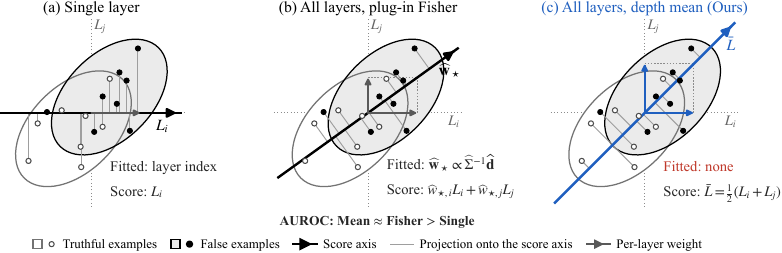}
\caption{\textbf{Three readouts of the same examples}.  A
single layer retains one coordinate.  The plug-in Fisher score estimates a
direction from the mean gap and equal-class pooled covariance.  The depth mean
uses fixed equal weights.  All three share the same training-fitted probes.}
\label{fig:readouts}
\vspace{-4mm}
\end{figure}

We find that averaging scores across all layers improves truthfulness prediction over a single layer selected using training data, consistent with recent evidence on multi-layer probing~\citep{nordby2026ensemble}. However, learning aggregation weights provides limited additional benefit. Our results show that the depth mean retains at least 91.7 percent of the held-out pooled signal-to-noise ratio (SNR) of the fitted Fisher score, with a difference of at most 0.72 points in the area under the receiver operating characteristic curve (AUROC). To understand this limited gain, we derive an exact Fisher-ratio decomposition in terms of class separation and cross-layer covariance. On the evaluated probe scores, it indicates that truthfulness information in layer-wise differences largely overlaps with that captured by the mean. With less training data, the unshrunk Fisher score deteriorates, while covariance shrinkage keeps its AUROC close to the depth mean. These findings motivate using equal weighting, which combines evidence across layers without estimating their covariance.

Guided by this analysis, we introduce \thistool{}, as shown in Figure~\ref{fig:readouts}(c), which trains layer-wise linear
probes and averages their logits to predict truthfulness before
decoding. Across six models and four benchmarks, \thistool{} achieves the
highest AUROC in 23 of 24 model--benchmark pairs and the highest AUPRC in 22
among the compared methods. Our contributions are as follows:
\begin{enumerate}[nosep,leftmargin=1.5em]
  \item \textbf{Aggregation formulation.}
  We distinguish the benefit of retaining more layers from the gain available
  through linear weighting and the effects of estimating those weights.

  \item \textbf{Uniform aggregation.}
  We derive an exact Fisher ratio decomposition relative to the depth mean and
  relate the resulting pooled-SNR gain to AUROC under stated score-difference
  conditions.

  \item \textbf{Experimental validation.}
  Across 18 conditions, we evaluate the gains from combining layers and learning
  their weights, and examine how the fitted scores change with training size.
  The depth mean retains at least 91.7 percent of the held-out pooled SNR of
  the fitted Fisher score.
\end{enumerate}

\section{Related Work}
\label{sec:relate}

\textbf{Hallucination Detection and Internal Probing.}
Methods operating on generated responses assess factuality through sampled
consistency~\citep{selfcheckgpt}, semantic or token-level
uncertainty~\citep{farquhar2024detecting,fadeeva2024tokenlevel}, and response
representations~\citep{chen2024inside,HalluGuard}.
Probing studies recover truthfulness signals from hidden activations and
attention heads~\citep{azaria2023internal,ccs,iti,marks2023geometry,fact-probe},
while IRIS~\citep{iris} learns from self-verification representations using
uncertainty-based pseudo labels. Prompt representations further support
hallucination-risk prediction before answer
generation~\citep{ji2024query,alnuhait2025factcheckmate,miftakhova2026softtarget,mrykhin2026pep}.
We build on this pre-decoding setting to examine how truthfulness evidence
from different layers can be combined.

\textbf{Cross-Layer Evidence and Aggregation.}
Intermediate prediction methods expose information across Transformer
depth~\citep{nostalgebraist2020logitlens,belrose2023tunedlens}.
DoLa~\citep{dola} and SLED~\citep{zhang2024sled} use contrasts between layers
to improve factuality during decoding. For detection, cross-layer information
is represented through logit statistics, semantic trajectories, or
information-flow measures~\citep{sun2025chair,lsd2025,kim2025layerwise}.
Aggregation ranges from fixed pooling to learned combinations.
EAGLE~\citep{xiao2026eagle} averages vocabulary logits across layers during
self-evaluation, while deception probes also average classifier probabilities
over selected layers~\citep{taufeeque2026obfuscation}.
Learned approaches combine features from multiple
layers~\citep{beigi2024internalinspector,srey2026uncertainty} or learn
combinations of selected layer-wise probes for deception
detection~\citep{nordby2026ensemble}.
Building on these approaches, we hold the layer-wise probe scores fixed to
distinguish the gain from retaining more layers from the additional benefit
of learning aggregation weights.

\textbf{Linear Aggregation and Weight Estimation.}
Fisher's criterion provides a classical reference for linear aggregation by
maximising pooled SNR~\citep{fisher1936use}. Research on diagnostic markers
studies linear combinations that maximise AUROC and the evaluation of fitted
combinations~\citep{su1993linear,huang2011optimal}.
Classifier-combination theory relates the ideal gain of weighted over simple
averaging to individual classifier performance and correlations between their
outputs~\citep{fumera2005theoretical}.
High-dimensional discriminant analysis characterises how covariance and
estimation error affect fitted
discriminants~\citep{bickel2004some,fan2012road,dobriban2018high}, motivating covariance
regularisation~\citep{ledoit2004well}.
Equal weighting can remain competitive when estimation error offsets the
potential gains from optimising
weights~\citep{dawes1979robust,smith2009simple,demiguel2009optimal}.
We study this trade-off in pre-decoding truthfulness detection, using an exact
Fisher-ratio decomposition to characterise the potential gain over the depth
mean and held-out evaluations to measure how much remains after fitting
weights on limited data.

\section{Method}
\label{sec:method}

\subsection{Overview}
\label{sec:overview}

Layer-wise probing supplies a score vector.  A detector turns that vector into
one scalar through two decisions: which coordinates to retain and how to
combine them.  Our method \thistool{} takes the full support and deploys the depth mean
(Figure~\ref{fig:readouts}). \S~\ref{sec:logit} instantiates the coordinates as answer-onset probe
logits, so the support mean is the deployed depth mean.  Proofs appear in
Appendix~\ref{app:snr-proof}.

\subsection{Score Vector, Support, and Support Mean}
\label{sec:aggregation}

Let a score-extraction procedure map an example to an indexed vector of scalar
scores $\mathbf{Z}=[Z_0,\ldots,Z_{m-1}]^\top\in\mathbb{R}^m$.
Here $m\geq1$ is the number of score coordinates and
$\mathcal{K}_m:=\{0,\ldots,m-1\}$ is their index set.  At evaluation time, the
extractor, coordinate standardisation, score orientation, and support are
fixed using the training data.  We use the label convention $Y=0$ for truthful
examples and $Y=1$ for false examples.

The analysis concerns linear scores
$\mathbf{Z}\mapsto\mathbf{w}^\top\mathbf{Z}$ with
$\mathbf{w}\in\mathbb{R}^m\setminus\{\mathbf{0}\}$.  The extraction procedure
fixes the meaning, ordering, and standardisation of the coordinates.  Here the
ordering is Transformer depth.  We interpret uniform weights relative to this
fixed coordinate system.

Fix a nonempty support $I\subseteq\mathcal{K}_m$ and let $\mathbf{1}_I$ be its
indicator vector.  The unit uniform direction and support mean are
\begin{equation}
  \label{eq:uniform-support}
  \mathbf{u}_I:=|I|^{-1/2}\mathbf{1}_I,
  \qquad
  \bar Z_I:=|I|^{-1}\sum_{j\in I}Z_j
  =|I|^{-1/2}\mathbf{u}_I^\top\mathbf{Z}.
\end{equation}
All scale-invariant comparisons below therefore give the same value for
$\mathbf{u}_I$ and $\bar Z_I$.

The support and the aggregation rule are separate decisions.
A singleton support retains one coordinate and discards the rest.  A larger
support contains more coordinates, after which uniform or fitted
weights determine how that evidence enters the scalar score.  The first choice
changes the available coordinates, and the second changes only their linear
combination.  The theoretical and empirical comparisons keep these two choices
separate.

We distinguish three effects on this score vector.  The support gain is the
held-out gain from enlarging the support under the uniform rule.  The
potential gain is the pooled-SNR gain of the Fisher-optimal direction over
the full-support mean, quantified by the Fisher ratio before aggregation
weights are fitted.  Estimation cost is the performance loss incurred by
estimating that direction from finite training data, assessed on held-out
folds.  All claims
concern linear aggregation of this fixed score vector.

\subsection{Fisher Ratio and AUROC}
\label{sec:detector}

We fix $I=\mathcal{K}_m$ to compare aggregation rules on the same full score
vector.  Write
$\boldsymbol{\mu}_c:=\mathbb{E}[\mathbf{Z}\mid Y=c]$ and
$\Sigma_c:=\operatorname{Cov}(\mathbf{Z}\mid Y=c)$.  The mean gap and the
\emph{equal-class} pooled covariance are
\begin{equation}
  \label{eq:pooled-population}
  \mathbf{d}:=\boldsymbol{\mu}_1-\boldsymbol{\mu}_0,
  \qquad
  \Sigma:=\frac{\Sigma_0+\Sigma_1}{2}.
\end{equation}
The equal class weights are part of the definition: they give
$\Sigma_0+\Sigma_1=2\Sigma$, which is the covariance identity used by the
pairwise AUROC result below.

For every nonzero $\mathbf{w}\in\mathbb{R}^m$, define the Fisher criterion
\begin{equation}
  \label{eq:fisher-criterion}
  J(\mathbf{w})
  :=\frac{(\mathbf{w}^\top\mathbf{d})^2}
          {\mathbf{w}^\top\Sigma\mathbf{w}}.
\end{equation}
Thus $J(\mathbf{w})$ is the squared pooled signal-to-noise ratio of the
linear score, while $\sqrt{J(\mathbf{w})}$ is its pooled SNR.  When
$\Sigma\succ0$, the Fisher-optimal direction is
$\mathbf{w}_\star:=\Sigma^{-1}\mathbf{d}$.
The criterion $J$ and the Gaussian relation between SNR and AUROC are
classical.  Their combination measures the gain available on the full support
deployed below.

To compare $\mathbf{w}_\star$ with the uniform direction, extend
$\mathbf{u}_I$ to an orthogonal matrix
$U_I=[\mathbf{u}_I,Q_I]$, where
$Q_I\in\mathbb{R}^{m\times(m-1)}$,
$Q_I^\top Q_I=\mathbf{I}_{m-1}$, and
$Q_I^\top\mathbf{u}_I=\mathbf{0}$.  Define the
coordinates along the uniform direction and its orthogonal complement
\begin{equation}
  \label{eq:uniform-blocks}
  U_I^\top\mathbf{d}
  =\begin{pmatrix}s\\\mathbf{z}\end{pmatrix},
  \qquad
  U_I^\top\Sigma U_I
  =\begin{pmatrix}a&\mathbf{b}^\top\\\mathbf{b}&C\end{pmatrix}.
\end{equation}
Here $s=\mathbf{u}_I^\top\mathbf{d}$ is the mean gap along the uniform
direction, $\mathbf{z}=Q_I^\top\mathbf{d}$ is its component in the orthogonal
complement, $a=\mathbf{u}_I^\top\Sigma\mathbf{u}_I$ is the variance along the
uniform direction, $\mathbf{b}=Q_I^\top\Sigma\mathbf{u}_I$ is the corresponding
cross-covariance, and $C=Q_I^\top\Sigma Q_I$ is the covariance on the
orthogonal complement.  When $\Sigma\succ0$, this change of coordinates gives
$a>0$, $C\succ0$, and
$a-\mathbf{b}^\top C^{-1}\mathbf{b}>0$.  Hence, for $s\neq0$, define
\begin{equation}
  \label{eq:fisher-diagnostics}
  \theta_I:=\frac{\mathbf{b}^\top C^{-1}\mathbf{b}}{a},
  \qquad
  \chi_I^2:=\frac{a\,\mathbf{z}^\top C^{-1}\mathbf{z}}{s^2},
  \qquad
  t_I:=\frac{\mathbf{b}^\top C^{-1}\mathbf{z}}{s},
\end{equation}
where $\chi_I\geq0$.  These scalars are invariant to the chosen orthonormal
basis $Q_I$ of the orthogonal complement.  When $m=1$, that complement is
zero-dimensional and $\theta_I=\chi_I=t_I=0$.

These quantities characterise the gain of Fisher weighting over the depth
mean.  $\chi_I$ measures class separation in the orthogonal
complement after covariance scaling.  $\theta_I$ measures the
normalised cross-covariance between the uniform direction and that complement,
while $t_I$ records their interaction.  The decomposition separates
additional mean difference from covariance-dependent reweighting.

\begin{theorem}[Fisher Ratio]
\label{thm:uniform-optimal}
Let the score vector $\mathbf{Z}\in\mathbb{R}^m$ be measurable
and have finite class-conditional second moments.  Fix the full support
$I=\mathcal{K}_m$.  Assume
$\Sigma\succ0$, $\mathbf{d}\neq\mathbf{0}$, and
$s=\mathbf{u}_I^\top\mathbf{d}\neq0$.  Define
\begin{equation}
  \Delta_I:=\sqrt{J(\mathbf{u}_I)},
  \qquad
  R_I:=\sqrt{\frac{J(\mathbf{w}_\star)}{J(\mathbf{u}_I)}}.
\end{equation}
Here $\Delta_I$ is the uniform pooled SNR and $R_I$ is the Fisher ratio: the
Fisher-optimal pooled SNR divided by $\Delta_I$.
Then $\mathbf{w}_\star$ maximises $J$ over all nonzero linear weights, and
the squared SNR ratio admits the decomposition
\begin{equation}
  \label{eq:fisher-gap-exact}
  R_I^2
  =\chi_I^2+\frac{(1-t_I)^2}{1-\theta_I},
  \qquad
  0\leq\theta_I<1,
  \qquad
  |t_I|\leq\sqrt{\theta_I}\,\chi_I.
\end{equation}
In particular, $R_I^2\geq1$.
\end{theorem}

When $m=1$, $R_I=1$, since the score vector fixes every nonzero
linear score up to scale.
The ratio $R_I$ compares the depth mean with all linear rules on the same
score vector.  Comparing the Fisher-optimal direction
with a fixed reference direction through a pooled-SNR ratio bounded below by
one is classical, with the independence rule as the reference
direction~\citep{fan2012road}.  Equation~\ref{eq:fisher-gap-exact} states that
comparison for the uniform direction.

To connect the Fisher criterion to ranking performance, take independent draws
$\mathbf{Z}_1\sim(\mathbf{Z}\mid Y=1)$ and
$\mathbf{Z}_0\sim(\mathbf{Z}\mid Y=0)$, and write
$D_{\mathbf{w}}:=\mathbf{w}^\top(\mathbf{Z}_1-\mathbf{Z}_0)$.  The tie-aware
pairwise AUROC is
$\mathbb{P}(D_{\mathbf{w}}>0)+\tfrac12\mathbb{P}(D_{\mathbf{w}}=0)$.
For $\Sigma\succ0$ and $\mathbf{w}\neq\mathbf{0}$, the signed pairwise SNR is
\begin{equation}
  \label{eq:pairwise-snr}
  \Gamma(\mathbf{w})
  :=\frac{\mathbb{E}[D_{\mathbf{w}}]}
          {\sqrt{\operatorname{Var}(D_{\mathbf{w}})}}
  =\frac{\mathbf{w}^\top\mathbf{d}}
          {\sqrt{2\mathbf{w}^\top\Sigma\mathbf{w}}}.
\end{equation}
When $\mathbf{w}^\top\mathbf{d}\geq0$, this becomes
$\Gamma(\mathbf{w})=\sqrt{J(\mathbf{w})/2}$.  The factor of two comes from
the variance of the difference between the two independent class draws.

\begin{corollary}[AUROC under Gaussian Score Differences]
\label{cor:gaussian-auroc}
Under the assumptions of Theorem~\ref{thm:uniform-optimal}, assume $s>0$.
Suppose, for
$\mathbf{w}\in\{\mathbf{u}_I,\mathbf{w}_\star\}$, that the one-dimensional
difference $D_{\mathbf{w}}$ follows the positive-variance law
$\mathcal{N}(\mathbf{w}^\top\mathbf{d},
2\mathbf{w}^\top\Sigma\mathbf{w})$.
Then, with $\Phi$ denoting the standard normal cumulative distribution
function,
\begin{equation}
  \label{eq:fisher-auroc-bridge}
  \operatorname{AUROC}(\mathbf{w})
  =\Phi(\Gamma(\mathbf{w}))
  =\Phi\!\left(\sqrt{\frac{J(\mathbf{w})}{2}}\right).
\end{equation}
In particular, the AUROC difference is
\begin{equation}
  \label{eq:auroc-gap-exact}
  \operatorname{AUROC}(\mathbf{w}_\star)
  -\operatorname{AUROC}(\mathbf{u}_I)
  =\Phi\!\left(\frac{R_I\Delta_I}{\sqrt2}\right)
   -\Phi\!\left(\frac{\Delta_I}{\sqrt2}\right),
\end{equation}
and satisfies
\begin{equation}
  \label{eq:auroc-gap}
  0\leq
  \operatorname{AUROC}(\mathbf{w}_\star)
  -\operatorname{AUROC}(\mathbf{u}_I)
  \leq
  \frac{(R_I-1)\Delta_I}{2\sqrt{\pi}}.
\end{equation}
\end{corollary}

The Gaussian condition is imposed on the two scalar differences.  The joint
score distribution may remain non-Gaussian, with unequal class covariances and
dependent coordinates.  The next result retains an explicit error term for an
approximate Gaussian AUROC formula.

\begin{corollary}[AUROC Bound with Gaussian Approximation Error]
\label{cor:approximate-auroc}
Under the assumptions of Theorem~\ref{thm:uniform-optimal}, let
$\kappa_\star,\kappa_I\geq0$ satisfy
\begin{equation}
  \left|
    \operatorname{AUROC}(\mathbf{w}_\star)
    -\Phi\!\left(\frac{R_I\Delta_I}{\sqrt2}\right)
  \right|
  \leq\kappa_\star,\quad
  \left|
    \operatorname{AUROC}(\mathbf{u}_I)
    -\Phi\!\left(\frac{\Delta_I}{\sqrt2}\right)
  \right|
  \leq\kappa_I.
\end{equation}
Then
\begin{equation}
  \label{eq:auroc-gap-approximate}
  \operatorname{AUROC}(\mathbf{w}_\star)
  -\operatorname{AUROC}(\mathbf{u}_I)
  \leq
  \frac{(R_I-1)\Delta_I}{2\sqrt\pi}
  +\kappa_\star+\kappa_I.
\end{equation}
\end{corollary}

The last inequality records deviations from the Gaussian AUROC formula through
the two explicit error terms.  The proofs appear in
Appendix~\ref{app:uniform-proof} and the two following subsections.
Appendix~\ref{rem:why-uniform} gives the train-fitted plug-in protocol and
separates its held-out score from the theoretical optimum.

The Fisher ratio isolates the multiplicative pooled-SNR gain, while the uniform
pooled SNR $\Delta_I$ sets the scale of the corresponding AUROC gap.
Appendix~\ref{app:uniform-proof} gives an
explicit weak bound and records the limiting cases of the decomposition.

\subsection{\thistool{} Probe-Logit Trajectory}
\label{sec:logit}

\thistool{} applies the preceding analysis to one score per layer.
Let $l\in\mathcal{K}_m$ index Transformer depth and let $p^\star$ be the
answer-onset position, namely the final prompt token immediately preceding
answer generation.  Let $H$ be the number of attention heads and $d_h$ their
width.  At layer $l$, we take the pre-output-projection attention vector at
$p^\star$ and partition it into the indexed family of head-sized slices
$(\mathbf{a}_l^{(h)})_{h=0}^{H-1}$, where
$\mathbf{a}_l^{(h)}\in\mathbb{R}^{d_h}$.  For each bootstrap resample of the
training fold, we fit regularised logistic probes to all heads at each layer.
We select the head with the highest training
AUROC, denoted by $k_l^\star$, giving the layer-$l$ readout
$\mathbf{a}_l:=\mathbf{a}_l^{(k_l^\star)}\in\mathbb{R}^{d_h}$.
Because $p^\star$ is determined by
the prompt, all coordinates are available before the first answer token is
emitted.

The selected probe uses coordinate-wise standardisation $\mathcal{S}_l$
fitted on the same resample.  Its decision logit
\begin{equation}
  \label{eq:logit}
  \ell_l
  :=\widetilde{\boldsymbol{\omega}}_l^\top\mathcal{S}_l(\mathbf{a}_l)
    +\widetilde b_l
\end{equation}
is the \emph{probe logit}, oriented so that larger values point toward
class~$1$.  We average these logits across bootstrap resamples to obtain the
layer score $L_l$.  The raw-coordinate parametrisation, head-selection rules,
and interpretation of the probe-logit scale are given in
Appendix~\ref{app:probe-training}.

The resulting score vector is the \emph{probe-logit trajectory}
\begin{equation}
  \label{eq:trajectory}
  \boldsymbol{\tau}
  :=[L_0,\ldots,L_{m-1}]^\top,
  \qquad \mathbf{Z}:=\boldsymbol{\tau}.
\end{equation}
Network depth supplies its ordering.  For the deployed detector, \thistool{}
takes the full support $I=\mathcal{K}_m$.  Since positive
rescaling changes neither $J$ nor pairwise AUROC, its score is written as the
arithmetic depth mean
\begin{equation}
  \label{eq:detector}
  \bar L
  :=\frac{1}{m}\sum_{l=0}^{m-1}L_l.
\end{equation}
One-dimensional logistic calibration is fitted and oriented within the
training fold.  Its decision logit is affine in $\bar L$, while the aggregation
weights remain fixed.  The retained layers determine the
support, and the depth mean determines the aggregation rule.  The experiments
vary the support, compare aggregation rules on the full probe-logit trajectory,
and evaluate the depth-mean detector.

\begin{remark}[Scale of the Probe-Logit Coordinates]
\label{rem:logit-scale}
The coordinates used by \thistool{} are bootstrap averages of logits from
training-fitted logistic probes, rather than raw activation magnitudes.  They
are dimensionless scores on a common fitted log-odds scale, so the depth mean
is defined in a common score coordinate system across layers.
Appendix~\ref{app:probe-training} records the corresponding product-rule
interpretation, the Gaussian equal-variance case, and the effect of coordinate
rescaling.
\end{remark}

By averaging probe logits across depth, \thistool{} combines truthfulness evidence from all layers before generation without estimating cross-layer aggregation weights. This design retains evidence that selecting a single layer would discard, while avoiding the weight-estimation cost identified by our analysis. \S~\ref{sec:experiments} evaluates the resulting depth-mean detector on held-out data.

\section{Experiments}
\label{sec:experiments}

\subsection{Experimental Setup}
\label{sec:exp-setup}

\textbf{Benchmarks and Backbones.}
The evaluation covers six LLMs and four benchmarks to assess the
effectiveness of depth averaging across model families, scales, and prediction
settings. The model suite comprises Llama-2-7B, Llama-3.1-8B, Qwen2.5-7B/14B/32B and Mistral-Small-24B.
HaluEval2~\citep{halueval2}, TrueFalse~\citep{azaria2023internal}, and
HELM~\citep{su2024mind} evaluate statement-level truth classification: the
statement appears in the prompt, and $Y$ is its released truth label,
following the probing settings of Inference-Time Intervention~\citep{iti}
and hidden-state truth classification~\citep{azaria2023internal}.
These benchmarks provide the 18 conditions for examining layer selection and
aggregation as defined in \S~\ref{sec:aggregation}, using the released labels
to measure pooled SNR and AUROC.
TruthfulQA~\citep{TruthfulQA} extends the evaluation to predicting factual
errors in the model's own answers. Following the pre-decoding setting in
\S~\ref{sec:logit}, \thistool{} computes its score before the answer is
generated. Each model then produces a response by greedy decoding, and GPT-5.4
evaluates it using the official TruthfulQA judging prompt, assigning $Y=1$
if the response contains any factual error. All methods are
evaluated against the same response labels for each model. The main comparison
covers all 24 conditions across the four benchmarks.

\textbf{Baselines.}
The comparison includes Fact-Probe~\citep{fact-probe},
ITI-Probe~\citep{iti}, IRIS~\citep{iris}, SIVR-PD~\citep{srey2026uncertainty},
and HalluGuard~\citep{HalluGuard}.  The first four construct classifiers from
internal states.  SIVR-PD adapts a state-of-the-art method for learned
multi-layer aggregation to answer onset. HalluGuard scores the complete
generated response after decoding.
We also compare against the plug-in Fisher score.

\textbf{Evaluation Protocol.}
We evaluate all methods using the same stratified five-fold split.
Detector fitting and feature selection use only the training folds.
The shrinkage analysis is a separate best-case comparison, reporting the
highest held-out AUROC across the tested strengths.

\begin{table}[t]
\centering
\caption{\textbf{Single-layer selection and full-support aggregation.}
Held-out pooled SNR and AUROC of the training-selected single layer (Single),
the plug-in Fisher score (Fisher), and the depth mean (Mean).  $j^\star$ is the
modal selected layer and $n_j$ the number of distinct layers selected across
the five folds.  Bold marks the largest AUROC in each row and underline the second
largest. }
\label{tab:single-layer-gap}
\scriptsize
\resizebox{0.8\textwidth}{!}{
\begin{tabular}{@{}ll@{\hspace{10pt}}rrr@{\hspace{10pt}}rrr@{\hspace{10pt}}rr@{}}
\toprule
 & & \multicolumn{3}{c}{Pooled SNR} & \multicolumn{3}{c}{AUROC (\%)} & \multicolumn{2}{c}{Selection} \\
\cmidrule(lr){3-5}\cmidrule(lr){6-8}\cmidrule(lr){9-10}
Model & Dataset & Single & Fisher & Mean & Single & Fisher & Mean & $j^\star$ & $n_j$ \\
\midrule
\multirow{3}{*}{Llama-2-7B}
 & HaluEval2 & 1.108 & 1.230 & 1.247 & 78.98 & \second{80.42} & \best{80.85} & 18 & 3 \\
 & TrueFalse & 2.330 & 2.548 & 2.479 & 94.29 & \best{95.23} & \second{95.19} & 18 & 1 \\
 & HELM & 1.389 & 1.687 & 1.664 & 83.55 & \best{87.89} & \second{87.70} & 12 & 3 \\
\addlinespace[4pt]
\multirow{3}{*}{Llama-3.1-8B}
 & HaluEval2 & 1.335 & 1.431 & 1.415 & 82.38 & \best{83.83} & \second{83.60} & 28 & 2 \\
 & TrueFalse & 2.841 & 3.147 & 3.038 & 96.27 & \second{97.19} & \best{97.30} & 21 & 2 \\
 & HELM & 1.515 & 1.754 & 1.753 & 85.64 & \second{88.66} & \best{88.80} & 13 & 1 \\
\addlinespace[4pt]
\multirow{3}{*}{Qwen2.5-7B}
 & HaluEval2 & 1.216 & 1.363 & 1.326 & 80.02 & \best{82.59} & \second{82.03} & 19 & 3 \\
 & TrueFalse & 2.978 & 3.321 & 3.070 & 96.84 & \best{97.47} & \second{97.43} & 19 & 1 \\
 & HELM & 1.400 & 1.669 & 1.667 & 83.83 & \second{87.69} & \best{87.85} & 4 & 5 \\
\addlinespace[4pt]
\multirow{3}{*}{Qwen2.5-14B}
 & HaluEval2 & 1.297 & 1.406 & 1.420 & 81.49 & \second{83.30} & \best{83.56} & 35 & 3 \\
 & TrueFalse & 3.328 & 3.758 & 3.512 & 97.79 & \second{98.25} & \best{98.34} & 44 & 3 \\
 & HELM & 1.456 & 1.719 & 1.713 & 84.48 & \second{88.14} & \best{88.34} & 45 & 4 \\
\addlinespace[4pt]
\multirow{3}{*}{Qwen2.5-32B}
 & HaluEval2 & 1.369 & 1.467 & 1.450 & 82.65 & \best{84.23} & \second{84.03} & 45 & 2 \\
 & TrueFalse & 3.250 & 3.897 & 3.573 & 97.51 & \second{98.49} & \best{98.51} & 60 & 1 \\
 & HELM & 1.375 & 1.727 & 1.740 & 83.36 & \second{88.32} & \best{88.51} & 42 & 4 \\
\addlinespace[4pt]
\multirow{3}{*}{Mistral-Small-24B}
 & HaluEval2 & 1.358 & 1.448 & 1.497 & 82.47 & \second{83.99} & \best{84.71} & 39 & 2 \\
 & TrueFalse & 3.192 & 3.724 & 3.590 & 97.54 & \second{98.13} & \best{98.22} & 36 & 4 \\
 & HELM & 1.528 & 1.819 & 1.813 & 85.76 & \second{89.27} & \best{89.50} & 18 & 3 \\
\midrule
\textbf{Average} & & 1.904 & 2.173 & 2.109 & 87.49 & \second{89.62} & \best{89.69} & & \\
\bottomrule
\end{tabular}
}
\vspace{-4mm}
\end{table}

\subsection{Gain from Aggregating Across Layers}
\label{sec:exp-depth}

We examine whether retaining more layers improves truthfulness prediction, using the same averaging rule in Equation~\ref{eq:uniform-support} for every layer subset. We compare subsets of two, four, and eight uniformly spaced layers with all layers and the first and second halves of the network. As shown in Figure~\ref{fig:depth-subsampling}, AUROC increases from two to eight layers in all nine conditions, and averaging over all layers matches or exceeds the eight-layer result. The stronger half varies across models, while the full-depth average outperforms both halves in every condition. These comparisons support collecting truthfulness evidence across the full network depth.
To assess whether a single informative layer can match this performance, we select the layer with the highest pooled SNR on each training fold. In Table~\ref{tab:single-layer-gap}, the depth mean achieves higher held-out AUROC than this selected layer in all 18 conditions, with gains of 0.55--5.15 points and pooled-SNR increases of 3.1\%--26.5\%. The largest gains occur on HELM, where all six models improve by 3.16--5.15 AUROC points. The gain persists both when the selected layer changes across folds and when it remains fixed.

\subsection{Aggregation on the Probe-Logit Trajectory}
\label{sec:exp-uniform}

\textbf{Limited gains over the depth mean.}
We examine whether learning aggregation weights improves prediction when the
full trajectory in Equation~\ref{eq:trajectory} is retained. We fit the Fisher
direction from Theorem~\ref{thm:uniform-optimal} and compare its held-out
performance with the depth mean. As shown in Table~\ref{tab:single-layer-gap},
the depth mean retains at least 91.7\% of the fitted Fisher score's pooled SNR
across all 18 conditions. Their AUROC differs by at most 0.72 points, and the
Fisher score averages 0.08 points below the mean. Using the Gaussian relation
in Equation~\ref{eq:fisher-auroc-bridge}, the measured SNR values predict AUROC
differences below one point, consistent with the observed differences.
On HaluEval2 and HELM, the predicted and observed gaps agree within 0.3 points.
The absolute approximation error $\widehat\kappa$ for individual scores ranges
from 0.2 to 1.6 AUROC points and is largest on TrueFalse, where AUROC exceeds
95 (Table~\ref{tab:fisher-diagnostics}).

\textbf{Redundancy across layers.}
To explain the limited gain from reweighting, we evaluate the decomposition
in Equation~\ref{eq:fisher-gap-exact}. In Table~\ref{tab:fisher-diagnostics}, the
estimated Fisher ratio $\widehat R_I$ ranges from 1.02 to 1.12, while
$\widehat\chi_I$ ranges from 0.56 to 1.09. The nonzero values of
$\widehat\chi_I$ show that differences between layer-wise scores carry class
separation. This separation is strongly coupled to the depth mean:
$\widehat t_I$ reaches at least 87\% of its Cauchy-Schwarz bound
$\sqrt{\widehat\theta_I}\,\widehat\chi_I$ in every condition.
Let $\widehat t_I=0$ in the decomposition while retaining the measured
$\widehat\theta_I$ and $\widehat\chi_I$ gives ratios of 1.33--4.93
(Appendix~\ref{rem:why-uniform}).
The measured coupling therefore explains why the class separation in
layer-wise differences adds little beyond the depth mean.

\textbf{Sensitivity to training size.}
We assess the data requirements of learned aggregation by re-estimating the
Fisher weights from smaller fractions of each training fold, keeping the
probes, depth mean, and held-out folds fixed. The Fisher rule estimates
$m(m{+}1)/2$ covariance entries, making covariance estimation sensitive to the
available training data~\citep{bickel2004some}. At full training size, the
ratio of covariance parameters to minority-class training examples ranges
from 0.16 to 1.74, and the training estimate $\widehat R_I$ exceeds the
held-out SNR ratio in all 18 conditions (Table~\ref{tab:fisher-diagnostics}).
As shown in Figure~\ref{fig:nmin-sweep}, reducing the training data to one
sixteenth increases the parameter ratio to 2.6--28.1 and the training-fold
$\widehat R_I$ to 1.12--1.59. At the same time, the unshrunk Fisher score falls
0.2--9.2 AUROC points below the depth mean in every condition. The growing
training ratio thus accompanies a loss in held-out ranking performance.
Under the best-case shrinkage comparison, the fitted score remains within
0.58 points of the depth
mean at one sixteenth of the training data (Table~\ref{tab:shrinkage-path}).
The selected strengths range from 0.15 to 1.00 at full training size and from
0.55 to 1.00 at one sixteenth. These comparisons support the depth mean as a
way to retain comparable AUROC without estimating additional aggregation
weights. Appendix~\ref{rem:why-uniform} gives the estimation and shrinkage
protocols.

\begin{figure}[t]
\centering
\includegraphics[width=\linewidth]{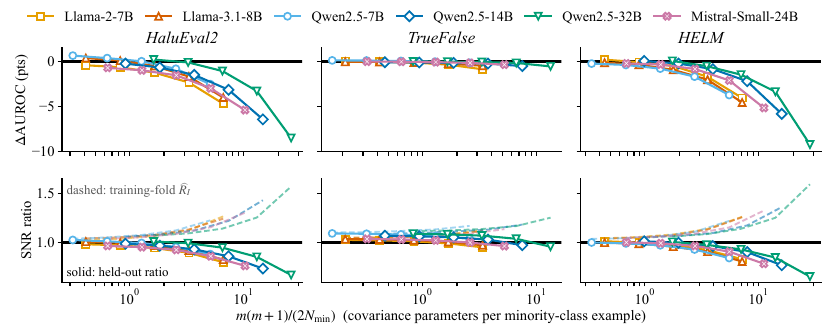}
\caption{\textbf{Estimation cost of the unshrunk plug-in Fisher score.}  The score is
re-estimated on a fraction of the training fold, with probes, held-out folds,
and the depth mean held fixed.  Top: held-out AUROC difference to the depth
mean.  Bottom: held-out SNR ratio (solid) and training-fold
$\widehat R_I$ (dashed).  Horizontal axis: covariance parameters per
minority-class training example, log scale.}
\label{fig:nmin-sweep}
\vspace{-4mm}
\end{figure}

\newcommand{\CanonicalQwenSevenTruthfulQAAUROC}{\second{75.38}$_{\pm1.3}$}
\newcommand{\CanonicalQwenSevenTruthfulQAAUPRC}{\second{63.95}$_{\pm3.2}$}
\newcommand{\CanonicalQwenFourteenTruthfulQAAUROC}{\second{79.67}$_{\pm3.9}$}
\newcommand{\CanonicalQwenFourteenTruthfulQAAUPRC}{58.80$_{\pm8.2}$}
\newcommand{\CanonicalQwenThirtyTwoTruthfulQAAUROC}{80.97$_{\pm2.6}$}
\newcommand{\CanonicalQwenThirtyTwoTruthfulQAAUPRC}{51.18$_{\pm3.1}$}
\newcommand{\CanonicalMistralSmallTruthfulQAAUROC}{\best{73.38}$_{\pm3.0}$}
\newcommand{\CanonicalMistralSmallTruthfulQAAUPRC}{\second{53.36}$_{\pm6.5}$}
\newcommand{\CanonicalLlamaTwoTruthfulQAAUROC}{\best{74.01}$_{\pm2.6}$}
\newcommand{\CanonicalLlamaTwoTruthfulQAAUPRC}{\best{73.71}$_{\pm3.7}$}
\newcommand{\CanonicalLlamaThreeTruthfulQAAUROC}{\second{71.24}$_{\pm4.4}$}
\newcommand{\CanonicalLlamaThreeTruthfulQAAUPRC}{\best{73.35}$_{\pm4.8}$}
\newcommand{\CanonicalQwenSevenHaluEvalAUROC}{\best{82.42}$_{\pm1.4}$}
\newcommand{\CanonicalQwenSevenHaluEvalAUPRC}{\best{78.14}$_{\pm0.3}$}
\newcommand{\CanonicalQwenSevenTrueFalseAUROC}{\best{97.59}$_{\pm0.3}$}
\newcommand{\CanonicalQwenSevenTrueFalseAUPRC}{\best{97.49}$_{\pm0.3}$}
\newcommand{\CanonicalQwenSevenHELMAUROC}{\best{87.84}$_{\pm1.2}$}
\newcommand{\CanonicalQwenSevenHELMAUPRC}{\best{88.50}$_{\pm1.1}$}
\newcommand{\CanonicalQwenFourteenHaluEvalAUROC}{\best{83.88}$_{\pm0.3}$}
\newcommand{\CanonicalQwenFourteenHaluEvalAUPRC}{\best{79.11}$_{\pm0.3}$}
\newcommand{\CanonicalQwenFourteenTrueFalseAUROC}{\best{98.40}$_{\pm0.3}$}
\newcommand{\CanonicalQwenFourteenTrueFalseAUPRC}{\best{98.23}$_{\pm0.3}$}
\newcommand{\CanonicalQwenFourteenHELMAUROC}{\best{88.34}$_{\pm0.4}$}
\newcommand{\CanonicalQwenFourteenHELMAUPRC}{\best{89.15}$_{\pm0.2}$}
\newcommand{\CanonicalQwenThirtyTwoHaluEvalAUROC}{\best{84.53}$_{\pm1.4}$}
\newcommand{\CanonicalQwenThirtyTwoHaluEvalAUPRC}{\best{79.39}$_{\pm1.3}$}
\newcommand{\CanonicalQwenThirtyTwoTrueFalseAUROC}{\best{98.54}$_{\pm0.3}$}
\newcommand{\CanonicalQwenThirtyTwoTrueFalseAUPRC}{\best{98.45}$_{\pm0.3}$}
\newcommand{\CanonicalQwenThirtyTwoHELMAUROC}{\best{88.59}$_{\pm1.3}$}
\newcommand{\CanonicalQwenThirtyTwoHELMAUPRC}{\best{89.23}$_{\pm1.0}$}
\newcommand{\CanonicalMistralSmallHaluEvalAUROC}{\best{84.71}$_{\pm1.3}$}
\newcommand{\CanonicalMistralSmallHaluEvalAUPRC}{\best{79.88}$_{\pm1.5}$}
\newcommand{\CanonicalMistralSmallTrueFalseAUROC}{\best{98.26}$_{\pm0.3}$}
\newcommand{\CanonicalMistralSmallTrueFalseAUPRC}{\best{98.11}$_{\pm0.4}$}
\newcommand{\CanonicalMistralSmallHELMAUROC}{\best{89.51}$_{\pm1.9}$}
\newcommand{\CanonicalMistralSmallHELMAUPRC}{\best{90.20}$_{\pm2.1}$}
\newcommand{\CanonicalLlamaTwoHaluEvalAUROC}{\best{80.68}$_{\pm1.2}$}
\newcommand{\CanonicalLlamaTwoHaluEvalAUPRC}{\best{76.25}$_{\pm1.1}$}
\newcommand{\CanonicalLlamaTwoTrueFalseAUROC}{\best{95.34}$_{\pm0.1}$}
\newcommand{\CanonicalLlamaTwoTrueFalseAUPRC}{\best{95.00}$_{\pm0.1}$}
\newcommand{\CanonicalLlamaTwoHELMAUROC}{\best{87.65}$_{\pm0.1}$}
\newcommand{\CanonicalLlamaTwoHELMAUPRC}{\best{88.06}$_{\pm0.4}$}
\newcommand{\CanonicalLlamaThreeHaluEvalAUROC}{\best{83.85}$_{\pm1.0}$}
\newcommand{\CanonicalLlamaThreeHaluEvalAUPRC}{\best{79.46}$_{\pm0.1}$}
\newcommand{\CanonicalLlamaThreeTrueFalseAUROC}{\best{97.44}$_{\pm0.4}$}
\newcommand{\CanonicalLlamaThreeTrueFalseAUPRC}{\best{97.26}$_{\pm0.6}$}
\newcommand{\CanonicalLlamaThreeHELMAUROC}{\best{88.73}$_{\pm1.0}$}
\newcommand{\CanonicalLlamaThreeHELMAUPRC}{\best{89.21}$_{\pm0.4}$}

\begin{table}[t]
\centering
\caption{%
  \textbf{Truthfulness prediction results (AUROC / AUPRC, \%) across four benchmarks and six LLMs.}
  \textbf{Bold}: largest point estimate; \underline{underline}: second largest.
  Where shown, $\pm$ denotes the standard deviation over five folds.
}
\label{tab:main-results}
\renewcommand{\arraystretch}{1.20}
\setlength{\tabcolsep}{3.2pt}
\resizebox{\textwidth}{!}{%
\begin{tabular}{@{} l l  cc  cc  cc  cc  cc  cc @{}}
\toprule
& & \multicolumn{2}{c}{\textbf{Qwen2.5-7B}}
& \multicolumn{2}{c}{\textbf{Qwen2.5-14B}}
& \multicolumn{2}{c}{\textbf{Qwen2.5-32B}}
& \multicolumn{2}{c}{\textbf{Mistral-Small-24B}}
& \multicolumn{2}{c}{\textbf{Llama-2-7B}}
& \multicolumn{2}{c}{\textbf{Llama-3.1-8B}} \\
\cmidrule(lr){3-4}\cmidrule(lr){5-6}\cmidrule(lr){7-8}\cmidrule(lr){9-10}\cmidrule(lr){11-12}\cmidrule(lr){13-14}
\textbf{Dataset} & \textbf{Method}
& AUROC & AUPRC & AUROC & AUPRC & AUROC & AUPRC & AUROC & AUPRC & AUROC & AUPRC & AUROC & AUPRC \\
\midrule
\multirow{6}{*}{\rotatebox[origin=c]{90}{\textsc{TruthfulQA}}}
& HalluGuard  & 60.57 & 31.95 & 59.84 & 35.33 & 52.25   & 18.21   & 53.82 & 34.05 & 54.98 & 46.25 & 51.60 & 42.96 \\
& Fact-Probe  & 67.26 & 48.40 & 60.30 & 22.00 & 75.53 & 40.14 & 69.23 & 47.79 & 69.20 & 62.98 & 66.07 & 59.33 \\
& ITI-Probe   & 64.46 & 44.71 & 69.92 & 42.95 & 74.57 & 42.72 & 65.95 & 46.82 & 64.90 & 56.09 & 63.43 & 57.04 \\
& IRIS        & 68.44 & 50.10 & 72.54 & 43.51 & 74.29 & 40.63 & 65.86 & 45.59 & 71.12 & 60.64 & 65.04 & 58.17 \\
& SIVR-PD     & \second{73.69} & \second{59.73} & \second{78.95} & \best{60.34} & \best{81.29} & \best{53.62} & \second{69.74} & \second{51.60} & \second{72.77} & \second{72.23} & \second{69.72} & \second{72.96} \\
& \thistool{}
  & \CanonicalQwenSevenTruthfulQAAUROC & \CanonicalQwenSevenTruthfulQAAUPRC
  & \CanonicalQwenFourteenTruthfulQAAUROC & \CanonicalQwenFourteenTruthfulQAAUPRC
  & \CanonicalQwenThirtyTwoTruthfulQAAUROC & \CanonicalQwenThirtyTwoTruthfulQAAUPRC
  & \CanonicalMistralSmallTruthfulQAAUROC & \CanonicalMistralSmallTruthfulQAAUPRC
  & \CanonicalLlamaTwoTruthfulQAAUROC & \CanonicalLlamaTwoTruthfulQAAUPRC
  & \CanonicalLlamaThreeTruthfulQAAUROC & \CanonicalLlamaThreeTruthfulQAAUPRC \\
\midrule

\multirow{6}{*}{\rotatebox[origin=c]{90}{\textsc{TrueFalse}}}
& HalluGuard  & 52.17 & 52.13 & 50.38 & 50.23 & 58.55 & 56.40 & 53.52 & 52.85 & 50.84 & 50.77 & 51.11 & 49.78 \\
& Fact-Probe  & 88.95 & 88.35 & 57.78 & 53.24 & 61.98 & 61.30   & 97.73 & 97.31 & 52.66 & 52.48 & 54.10 & 53.68 \\
& ITI-Probe   & 96.98 & 96.76 & \second{98.08} & 97.87 & 98.05 & 97.98 & \second{98.01} & 97.83 & 94.59 & 94.19 & \second{96.97} & 96.69 \\
& IRIS        & 95.36 & 94.41 & 94.86 & 90.38 & 97.78 & 97.53 & 97.28 & 96.54 & 93.94 & 93.49 & 96.73 & 96.42 \\
& SIVR-PD     & \second{97.27} & \second{97.41} & 97.85 & \second{97.91} & \second{98.19} & \second{98.27} & 97.82 & \second{98.00} & \second{94.67} & \second{94.88} & 96.64 & \second{96.86} \\
& \thistool{}
  & \CanonicalQwenSevenTrueFalseAUROC & \CanonicalQwenSevenTrueFalseAUPRC
  & \CanonicalQwenFourteenTrueFalseAUROC & \CanonicalQwenFourteenTrueFalseAUPRC
  & \CanonicalQwenThirtyTwoTrueFalseAUROC & \CanonicalQwenThirtyTwoTrueFalseAUPRC
  & \CanonicalMistralSmallTrueFalseAUROC & \CanonicalMistralSmallTrueFalseAUPRC
  & \CanonicalLlamaTwoTrueFalseAUROC & \CanonicalLlamaTwoTrueFalseAUPRC
  & \CanonicalLlamaThreeTrueFalseAUROC & \CanonicalLlamaThreeTrueFalseAUPRC \\
\midrule

\multirow{6}{*}{\rotatebox[origin=c]{90}{\textsc{HaluEval2}}}
& HalluGuard  & 50.46 & 45.91 & 53.19 & 47.73 & 51.29 & 47.46 & 55.66 & 50.97 & 50.48 & 45.63 & 50.61 & 46.26 \\
& Fact-Probe  & 57.13 & 53.73 & 58.20 & 54.49 &  62.37 & 57.73  & 77.29 & 70.18 & 58.63 & 55.61 & 59.71 & 56.26 \\
& ITI-Probe   & \second{81.42} & 76.31 & 82.34 & 76.37 & \second{82.98} & 76.87 & 83.04 & 77.18 & 79.11 & 74.10 & \second{82.24} & \second{77.39} \\
& IRIS        & 79.96 & 74.68 & 81.78 & 74.53 & 81.83 & \second{77.98} & 84.02 & \second{79.23} & 78.86 & \second{74.31} & 81.43 & 75.59 \\
& SIVR-PD     & 80.85 & \second{76.32} & \second{82.98} & \second{77.72} & 82.54 & 76.77 & \second{84.12} & 78.85 & \second{79.27} & 74.15 & 81.55 & 75.98 \\
& \thistool{}
  & \CanonicalQwenSevenHaluEvalAUROC & \CanonicalQwenSevenHaluEvalAUPRC
  & \CanonicalQwenFourteenHaluEvalAUROC & \CanonicalQwenFourteenHaluEvalAUPRC
  & \CanonicalQwenThirtyTwoHaluEvalAUROC & \CanonicalQwenThirtyTwoHaluEvalAUPRC
  & \CanonicalMistralSmallHaluEvalAUROC & \CanonicalMistralSmallHaluEvalAUPRC
  & \CanonicalLlamaTwoHaluEvalAUROC & \CanonicalLlamaTwoHaluEvalAUPRC
  & \CanonicalLlamaThreeHaluEvalAUROC & \CanonicalLlamaThreeHaluEvalAUPRC \\
\midrule

\multirow{6}{*}{\rotatebox[origin=c]{90}{\textsc{HELM}}}
& HalluGuard  & 52.18 & 51.61 & 51.89 & 49.08 & 51.56 & 55.09 & 51.47 & 54.69 & 51.61 & 54.27 & 54.31 & 56.72 \\
& Fact-Probe  & 59.14 & 59.89 & 57.78 & 59.55 & 67.76	& 67.00   & 86.78 & 86.75 & 57.53 & 59.40 & 58.14 & 58.94 \\
& ITI-Probe   & 85.45 & 86.48 & 85.54 & 85.51 & 86.07 & 86.49 & 87.22 & \second{87.63} & 84.25 & 84.50 & 80.38 & 80.05 \\
& IRIS        & \second{86.08} & \second{86.49} & \second{87.40} & \second{87.33} & \second{87.45} & \second{87.82} & \second{88.09} & 87.42 & \second{87.15} & \second{86.69} & 79.51 & 80.38 \\
& SIVR-PD     & 82.89 & 83.85 & 83.51 & 83.17 & 83.10 & 83.26 & 83.51 & 83.14 & 82.32 & 82.66 & \second{84.24} & \second{83.75} \\
& \thistool{}
  & \CanonicalQwenSevenHELMAUROC & \CanonicalQwenSevenHELMAUPRC
  & \CanonicalQwenFourteenHELMAUROC & \CanonicalQwenFourteenHELMAUPRC
  & \CanonicalQwenThirtyTwoHELMAUROC & \CanonicalQwenThirtyTwoHELMAUPRC
  & \CanonicalMistralSmallHELMAUROC & \CanonicalMistralSmallHELMAUPRC
  & \CanonicalLlamaTwoHELMAUROC & \CanonicalLlamaTwoHELMAUPRC
  & \CanonicalLlamaThreeHELMAUROC & \CanonicalLlamaThreeHELMAUPRC \\
\bottomrule
\end{tabular}%
}
\vspace{-3mm}
\end{table}

\subsection{Main Results}
\label{sec:exp-detection}

\textbf{Overall detection performance.}
We evaluate whether the depth mean in Equation~\ref{eq:detector} yields competitive
truthfulness prediction across model families and benchmarks. As shown in
Table~\ref{tab:main-results}, \thistool{} achieves the highest AUROC in 23 of
24 conditions and the highest AUPRC in 22 among the compared methods.
The largest margins over the strongest baseline are 4.49 AUROC points and
5.46 AUPRC points, both on HELM with Llama-3.1-8B. These results complement
the controlled aggregation comparisons in \S~\ref{sec:exp-depth}
and~\ref{sec:exp-uniform}, showing that the resulting detector is competitive
across the evaluated model families and benchmarks.

\textbf{Prediction before generation.}
On TruthfulQA, HalluGuard~\citep{HalluGuard} scores the complete generated
response, whereas \thistool{} predicts response truthfulness from activations
available before decoding (\S~\ref{sec:logit}). In
Table~\ref{tab:main-results}, \thistool{} achieves higher AUROC and AUPRC
than HalluGuard on all six models, with AUROC gains of 14.81--28.72 points.
These results support the depth mean as an early indicator of response
factuality.

\subsection{Ablations and Controls}
\label{sec:exp-ablation}

\textbf{Trajectory Features and Aggregation.}
We test whether richer features or alternative aggregation rules improve
prediction beyond the depth mean in Equation~\ref{eq:detector}. The feature
comparison includes classifiers using trajectory statistics, increment
features, or an expanded input containing all layer-wise probe logits,
their first differences, and six summary statistics (Appendix~\ref{app:features}).
As shown in Table~\ref{tab:four-way-auroc}, the largest improvement over the
mean is 0.6 AUROC points, obtained with the trajectory statistics, while the
classifier using the expanded input can fall up to 1.0 point below it.
Replacing the mean with the median or a 10\%-trimmed mean changes AUROC by
at most 0.3 points across nine conditions (Appendix~\ref{app:mean-vs-median}).
At full training size, the best-case shrinkage comparison yields a gain of
at most 0.79 points over the mean, averaging
0.24 points across the 18 conditions (Table~\ref{tab:shrinkage-path}).
These comparisons complement the linear weighting analysis in
\S~\ref{sec:exp-uniform} and support the depth mean as an effective detector
input, with limited gains from the tested feature expansions and aggregation
alternatives.

\textbf{Readout Controls.}
We examine how prediction depends on the head selection, internal
representation, and probe-logit scale used in \S~\ref{sec:logit}.
In the head-selection control, which averages probe probabilities,
Random-1 remains within 0.7 AUROC points of Top-1 across nine conditions
(Table~\ref{tab:topk_heads}). Using residual-stream activations with the same
depth mean changes AUROC by at most 0.5 points on HaluEval2 and TrueFalse,
and improves HELM by 1.8--2.3 points across five models
(Table~\ref{tab:readout_ablation}). The HELM gains show that changing the
internal representation can improve prediction while retaining the same
aggregation rule. The largest AUROC gain from standardising the probe logits
or averaging their sigmoid-transformed probabilities is 0.18 points across
18 conditions (Table~\ref{tab:score-coordinate-scale}), supporting the use of
the fitted logit scale described in Remark~\ref{rem:logit-scale}.
A classifier using statement length alone reaches at most 55.5 AUROC on the
three statement benchmarks, at least 25 points below the depth-mean detector
(Appendix~\ref{app:probe-shortcut}).


\section{Conclusion}
\label{sec:conclusion}

Reliable use of language models depends on identifying answers that need
further factual verification. We study how to provide this signal before
generation by aggregating truthfulness evidence from internal states.
Retaining more layers improves prediction even with fixed equal weights,
while our analysis explains the limited reweighting gain through overlap
between the depth mean and layer-wise differences. Across 18 conditions,
the depth mean remains within 0.72 AUROC points of the fitted Fisher linear
discriminant, without learning aggregation weights. This finding motivates
\thistool{}, which averages layer-wise probe logits to predict truthfulness
before decoding. Across six models and four benchmarks, it achieves the
highest AUROC among the compared methods in 23 of 24 conditions. On TruthfulQA, it
achieves higher AUROC and AUPRC than HalluGuard on all six models, while
HalluGuard uses the complete generated response. By making a factuality score
available at answer onset, this approach could inform whether additional
verification is needed before the model generates an answer.

\bibliographystyle{plainnat}
\bibliography{refs}

\appendix

\section{Implementation Details}
\label{app:implementation}

This section specifies how we standardise head activations, fit layer-wise
probes, and select heads to construct the probe-logit trajectory in
Equation~\ref{eq:trajectory}.  Its depth mean is the input to the scalar
logistic classifier.

\paragraph{Training settings.}
Layer-wise probes use $L_2$-regularised logistic regression with \texttt{C=1} and the
\texttt{liblinear} solver.  We use three bootstrap resamples per training fold,
each drawn uniformly with replacement at the original fold size.  The logits
are averaged before fitting a single scalar calibration model.

\subsection{Probe Training and Head Selection}
\label{app:probe-training}

For an autoregressive Transformer with $m$ layers and $H$ attention heads,
let $\mathbf{a}_l^{(h)}\in\mathbb{R}^{d_h}$ denote the raw
pre-output-projection attention slice at the answer-onset position for
$l\in\mathcal{K}_m$ and $h\in\{0,\ldots,H-1\}$.

\paragraph{Supervised projection.}
For each bootstrap member and candidate pair $(l,h)$, the resample supplies a coordinate-wise
mean $\widehat{\boldsymbol\mu}_{l,h}$ and positive scale
$\widehat{\boldsymbol\sigma}_{l,h}$.  The fixed standardisation map is
\begin{equation}
  \label{eq:standardize}
  \mathcal{S}_{l,h}(\mathbf{a})
  :=
  \operatorname{diag}(\widehat{\boldsymbol\sigma}_{l,h})^{-1}
  (\mathbf{a}-\widehat{\boldsymbol\mu}_{l,h}).
\end{equation}
A regularised logistic regression fitted in this coordinate system has
decision score, for $\widetilde{\mathbf{a}}\in\mathbb{R}^{d_h}$,
\begin{equation}
  \label{eq:standardized-probe}
  \widetilde{\phi}_{l,h}(\widetilde{\mathbf{a}})
  =
  \widetilde{\boldsymbol{\omega}}_{l,h}^{\!\top}\widetilde{\mathbf{a}}
  +\widetilde b_{l,h}.
\end{equation}
Composing the two maps gives the same score in raw activation coordinates:
\begin{equation}
  \label{eq:raw-space-probe}
  \widehat{\phi}_{l,h}(\mathbf{a})
  =
  \widehat{\boldsymbol{\omega}}_{l,h}^{\!\top}\mathbf{a}
  +\widehat b_{l,h},
  \qquad
  \widehat{\boldsymbol{\omega}}_{l,h}
  =
  \operatorname{diag}(\widehat{\boldsymbol\sigma}_{l,h})^{-1}
  \widetilde{\boldsymbol{\omega}}_{l,h},
\end{equation}
\begin{equation}
  \widehat b_{l,h}
  =
  \widetilde b_{l,h}
  -\widehat{\boldsymbol{\omega}}_{l,h}^{\!\top}
   \widehat{\boldsymbol\mu}_{l,h}.
\end{equation}
After selecting $k_l^\star$, we abbreviate
\begin{equation}
  \begin{aligned}
  \mathcal{S}_l&:=\mathcal{S}_{l,k_l^\star},
  &\widetilde{\boldsymbol{\omega}}_l&:=\widetilde{\boldsymbol{\omega}}_{l,k_l^\star},
  &\widetilde b_l&:=\widetilde b_{l,k_l^\star},\\
  \widehat{\boldsymbol{\omega}}_l&:=\widehat{\boldsymbol{\omega}}_{l,k_l^\star},
  &\widehat b_l&:=\widehat b_{l,k_l^\star},
  \end{aligned}
\end{equation}
and use
\begin{equation}
  \phi_l(\mathbf{a})
  :=\widetilde{\boldsymbol{\omega}}_l^{\!\top}\mathcal{S}_l(\mathbf{a})
    +\widetilde b_l
  =\widehat{\phi}_{l,k_l^\star}(\mathbf{a}),
  \qquad
  \ell_l:=\phi_l(\mathbf{a}_l^{(k_l^\star)}).
\end{equation}
These unnormalised probe logits are averaged across bootstrap members to
form $L_l$.  Their orientation follows the binary probe convention $Y=1$
for false examples.
The corresponding unit raw-activation direction is
\begin{equation}
  \label{eq:raw-space-normalized-probe}
  \mathbf{v}_{l,h}
  :=
  \frac{\widehat{\boldsymbol{\omega}}_{l,h}}
       {\|\widehat{\boldsymbol{\omega}}_{l,h}\|_2},
  \qquad
  \mathbf{v}_l:=\mathbf{v}_{l,k_l^\star}.
\end{equation}
The two scales serve different roles: $L_l$ is the score coordinate used by
the detector, and $\mathbf{v}_l$ is the unit direction used by the
probe-alignment feature in Appendix~\ref{app:features}.

\paragraph{Head selection protocol.}
In the main experiments, we fit a logistic probe for every head at each layer
and select the head with the highest training AUROC as $k_l^\star$.
The head-selection ablation uses the same AUROC ranking for its Top-$K$ and
Bottom-1 controls.  Head choices are fixed before the held-out fold is scored.
An alternative fast implementation selects the head with the largest raw
class-mean separation,
\begin{equation}
  g_{l,h}
  :=
  \left\|
    \mathbb{E}_{\mathrm{tr}}[\mathbf{a}_l^{(h)}\mid Y=1]
    -
    \mathbb{E}_{\mathrm{tr}}[\mathbf{a}_l^{(h)}\mid Y=0]
  \right\|_2.
\end{equation}

\paragraph{Bootstrap probe average.}
The main experiments use bootstrap probe averaging.  Each bootstrap resample
of the training fold produces a trajectory with its own head choices and
standardisers, and the logits are averaged at each layer.  Because averaging
is linear, the $W/K$ identity holds for every ensemble member and for their
coordinate-wise average.

\paragraph{Scale of the probe-logit coordinates.}
Each probe logit is the training-fitted logistic decision score at that layer,
a dimensionless quantity independent of the activation scale of the selected
head.  Logistic regression places these scores on a log-odds scale
~\citep{niculescu2005predicting}.  Their equal-weight sum is the product-rule
combination of the layer-wise readouts~\citep{kittler1998combining}.  Under a
Gaussian model with equal class variances, a calibrated log-odds coordinate
satisfies $d_l=\Sigma_{ll}$~\citep{good1979turing}.  If the coordinates are also
uncorrelated, the uniform and Fisher-optimal directions coincide.  More
generally, the Fisher ratio measures the pooled-SNR gain available from the
remaining mean and covariance structure.  A coordinate-wise rescaling $D$
changes the uniform reference direction to $D\mathbf{1}$.

\section{Head Selection Sensitivity}
\label{app:head-uniformity}

This section measures how detection changes with the selected attention head
through exhaustive per-head probing and multi-head averaging.

\paragraph{Per-Head Discriminability.}
For each of the $m\times H$ layer-head coordinates, we train an independent
logistic regression probe and record its training AUROC.
Table~\ref{tab:head_distribution} reports the results across three models.
The per-head AUROC has CV below $0.08$ across the three models.

Raw head-wise class-mean separation is scale-sensitive.  Every probe is fitted
after head-specific standardisation (Equation~\ref{eq:standardize}), so per-head
AUROC reflects separation relative to the fitted coordinate scale.
On the tested HaluEval2 representation, the small AUROC dispersion is
consistent with separation distributed across several heads.

\begin{table}[h]
\centering
\caption{Head-level probe statistics on pooled HaluEval2.
AUC CV: within-layer AUROC coefficient of variation;
$\bar r_{\mathrm{head}}$: mean pairwise Pearson correlation of head probabilities;
Eff.\ Rank: exponential entropy of normalised correlation eigenvalues.
CV, correlation, and effective rank are averaged over layers.}
\label{tab:head_distribution}
\small
\setlength{\tabcolsep}{4pt}
\begin{tabular}{lccccc}
\toprule
\textbf{Model} & $m{\times}H$ & \textbf{AUC CV} & $\boldsymbol{\bar r_{\mathrm{head}}}$ & \textbf{Eff.\ Rank} & \textbf{Rank/}$H$ \\
\midrule
Qwen2.5-7B   & $28{\times}28$ & 0.047 & 0.75 & 3.5 & 12.4\% \\
Llama-2-7B   & $32{\times}32$ & 0.073 & 0.65 & 5.7 & 17.8\% \\
Llama-3.1-8B & $32{\times}32$ & 0.061 & 0.74 & 3.7 & 11.5\% \\
\bottomrule
\end{tabular}
\end{table}

\paragraph{Multi-Head Ensemble and Random-Head Baseline.}
For this selection ablation, let $P_l^{(h)}$ denote the class-$1$ probability
returned by the probe for head $h$ at layer $l$.  The evaluated score averages
the selected probabilities within each layer and then across depth.  This
probability trajectory serves as a head-sensitivity control, and the deployed
detector uses the probe-logit trajectory.
If the head readouts are largely redundant, Top-5 averaging should add little
to Top-1 selection, and random selection should remain close to the ranked
choice.  Table~\ref{tab:topk_heads} reports that Top-1 and Top-5 differ by at
most $0.1$ percentage point, and Random-1 lies within $0.7$ percentage point
of Top-1 across the nine conditions.  Bottom-1 remains above chance with a
larger, more variable gap.

The small Top-5 gain and the narrow Random-1/Top-1 gap indicate overlap among
the tested within-layer readouts.
To quantify this redundancy, we compute the pairwise Pearson correlation of within-layer probe probabilities and the effective rank of the resulting correlation matrix (Table~\ref{tab:head_distribution}).
Across the three reported models, the mean probability correlation is at least
$0.65$ and the effective rank is at most $18\%$ of $H$.

The two analyses report a modest head-choice difference in the evaluated
probability depth means and a larger depth-support difference
(Figure~\ref{fig:depth-subsampling}).

\begin{table}[t]
\centering
\caption{Head selection with probability averaging within and across layers.
Top-$K$ and Bottom-1 follow the per-layer training AUROC ranking.
Random-1 samples one head per layer uniformly (mean $\pm$ std over 50 draws).
AUROC means are in percent; $\Delta$ is Random-1 minus Top-1 in AUROC points.}
\label{tab:topk_heads}
\small
\setlength{\tabcolsep}{4pt}
\begin{tabular}{llccccc}
\toprule
\textbf{Model} & \textbf{Dataset}
  & \textbf{Top-1} & \textbf{Top-5}
  & \textbf{Bottom-1} & \textbf{Random-1}
  & $\boldsymbol{\Delta}$\textbf{(pp)} \\
\midrule
\multirow{3}{*}{Llama-2-7B}
  & HaluEval2  & 80.8 & 80.9 & 78.7 & 80.4{\scriptsize$\pm$.013} & $-0.4$ \\
  & TrueFalse & 95.4 & 95.4 & 91.7 & 94.7{\scriptsize$\pm$.004} & $-0.7$ \\
  & HELM      & 88.3 & 88.4 & 87.0 & 88.0{\scriptsize$\pm$.016} & $-0.3$ \\
\midrule
\multirow{3}{*}{Llama-3.1-8B}
  & HaluEval2  & 83.9 & 83.9 & 82.0 & 83.3{\scriptsize$\pm$.012} & $-0.5$ \\
  & TrueFalse & 97.5 & 97.5 & 96.2 & 97.2{\scriptsize$\pm$.004} & $-0.3$ \\
  & HELM      & 89.2 & 89.3 & 88.1 & 88.9{\scriptsize$\pm$.016} & $-0.4$ \\
\midrule
\multirow{3}{*}{Qwen2.5-7B}
  & HaluEval2  & 82.2 & 82.2 & 81.1 & 81.9{\scriptsize$\pm$.014} & $-0.3$ \\
  & TrueFalse & 97.4 & 97.4 & 96.2 & 97.1{\scriptsize$\pm$.004} & $-0.3$ \\
  & HELM      & 88.4 & 88.4 & 87.9 & 88.0{\scriptsize$\pm$.015} & $-0.4$ \\
\bottomrule
\end{tabular}
\end{table}

\paragraph{Residual-Stream Readout.}
We fit layer-wise probes to the answer-onset residual-stream states and
aggregate their logits with the same depth mean.  This comparison holds the
depth support and aggregation rule fixed while changing the layer
representation (Table~\ref{tab:readout_ablation}).

\begin{table}[H]
\centering
\caption{Readout mechanism ablation (AUROC): selected-head readout ($d{=}d_h$) vs.\ residual-stream ($d{=}d_{\mathrm{model}}$, no head selection), both evaluated via the depth mean $\bar{L}$.
$\Delta$: residual-stream minus selected-head.
The two readouts differ by at most $0.005$ on HaluEval2 and TrueFalse, and the residual-stream value is higher on HELM.}
\label{tab:readout_ablation}
\small
\setlength{\tabcolsep}{5pt}
\begin{tabular}{llccc}
\toprule
\textbf{Model} & \textbf{Dataset} & \textbf{Selected Head} & \textbf{Residual Stream} & $\boldsymbol{\Delta}$ \\
\midrule
\multirow{3}{*}{Qwen2.5-32B}
  & HaluEval2  & 0.852 & 0.856 & $+.005$ \\
  & TrueFalse & 0.985 & 0.986 & $+.001$ \\
  & HELM      & 0.887 & 0.907 & $+.020$ \\
\midrule
\multirow{3}{*}{Qwen2.5-14B}
  & HaluEval2  & 0.847 & 0.847 & $\phantom{+}.000$ \\
  & TrueFalse & 0.983 & 0.984 & $+.001$ \\
  & HELM      & 0.884 & 0.907 & $+.023$ \\
\midrule
\multirow{3}{*}{Qwen2.5-7B}
  & HaluEval2  & 0.832 & 0.829 & $-.003$ \\
  & TrueFalse & 0.974 & 0.975 & $+.001$ \\
  & HELM      & 0.877 & 0.897 & $+.021$ \\
\midrule
\multirow{3}{*}{Llama-3.1-8B}
  & HaluEval2  & 0.846 & 0.848 & $+.002$ \\
  & TrueFalse & 0.973 & 0.976 & $+.003$ \\
  & HELM      & 0.888 & 0.906 & $+.018$ \\
\midrule
\multirow{3}{*}{Llama-2-7B}
  & HaluEval2  & 0.821 & 0.822 & $+.001$ \\
  & TrueFalse & 0.952 & 0.956 & $+.004$ \\
  & HELM      & 0.877 & 0.896 & $+.019$ \\
\bottomrule
\end{tabular}
\end{table}

\section{Complete Feature Definitions}
\label{app:features}

This section defines the statistics used in the ablations: the depth mean
$\bar L$, four additional trajectory statistics, and 13 statistics derived
from the $W/K$ increment decomposition.  Table~\ref{tab:four-way-auroc}
reports the comparison of \S~\ref{sec:exp-ablation} on all 18
conditions.
All feature ablations use the same training and evaluation settings as
\thistool{}.

\begin{table}[t]
\centering
\caption{\textbf{Aggregation inputs on a fixed probe-logit trajectory} under
one shared full-support, five-fold run (AUROC, \%).  Mean: depth mean (1D);
Traj.: trajectory statistics (5D); Incr.: increment statistics (13D);
Exp.: expanded trajectory ($2m{+}5$D).  The depth mean stays within 1.1 AUROC
points of every tested richer representation.  Bold marks the largest
displayed point estimate within each benchmark.}
\label{tab:four-way-auroc}
\small
\setlength{\tabcolsep}{3pt}
\begin{tabular}{l *{4}{c} @{\hspace{9pt}} *{4}{c} @{\hspace{9pt}} *{4}{c}}
\toprule
& \multicolumn{4}{c}{\textbf{HaluEval2}} & \multicolumn{4}{c}{\textbf{TrueFalse}} & \multicolumn{4}{c}{\textbf{HELM}} \\
\cmidrule(lr){2-5} \cmidrule(lr){6-9} \cmidrule(lr){10-13}
\textbf{Model} & Mean & Traj. & Incr. & Exp. & Mean & Traj. & Incr. & Exp. & Mean & Traj. & Incr. & Exp. \\
\midrule
Mistral-Small-24B & \best{84.7} & \best{84.7} & 84.2 & 83.7 & 98.2 & \best{98.4} & \best{98.4} & \best{98.4} & \best{89.5} & \best{89.5} & 89.0 & 89.2 \\
Qwen2.5-32B & 84.0 & \best{84.6} & 84.3 & 83.8 & 98.5 & \best{98.6} & \best{98.6} & 98.3 & 88.5 & \best{88.8} & 88.4 & 88.7 \\
Qwen2.5-14B & 83.6 & \best{83.8} & \best{83.8} & 83.1 & 98.3 & \best{98.4} & 98.3 & 98.0 & 88.3 & 88.4 & \best{88.5} & 88.3 \\
Qwen2.5-7B & 82.0 & \best{82.4} & 82.2 & \best{82.4} & 97.4 & \best{97.5} & \best{97.5} & 97.2 & \best{87.9} & 87.7 & 87.6 & 87.6 \\
Llama-3.1-8B & 83.6 & \best{83.8} & 83.3 & \best{83.8} & 97.3 & 97.4 & \best{97.5} & 97.1 & 88.8 & 88.7 & 88.3 & \best{88.9} \\
Llama-2-7B & \best{80.9} & 80.8 & 80.7 & 80.1 & 95.2 & \best{95.3} & \best{95.3} & \best{95.3} & 87.7 & 87.7 & 87.3 & \best{87.8} \\
\bottomrule
\end{tabular}
\vspace{-4mm}
\end{table}

\subsection{Trajectory Statistics}
Given the trajectory
$\boldsymbol{\tau}=[L_0,\ldots,L_{m-1}]^\top$, the ablations use a fixed
middle window from one third to three quarters of depth and a late window
containing the remaining layers.  The deployed detector uses the depth mean
$\bar L$, independently of these windows.  The five statistics are:
\begin{enumerate}
    \item \textbf{Mean} ($\bar L$): $m^{-1}\sum_{l=0}^{m-1}L_l$.
    \item \textbf{OLS slope} ($\beta$): the ordinary least-squares coefficient
    of $L_l$ on the depth index $l$.
    \item \textbf{Window mean difference} ($\Delta_{\mathrm{window}}$):
    $\bar L_{\mathrm{late}}-\bar L_{\mathrm{mid}}$.
    \item \textbf{Total variation} ($\mathrm{TV}$):
    $\sum_{l=0}^{m-2}|\Delta L_l|$, where $\Delta L_l:=L_{l+1}-L_l$.
    \item \textbf{Range}: $\max_lL_l-\min_lL_l$.
\end{enumerate}
\subsection{Increment-Feature Definition}
Within each bootstrap member, the frozen cross-layer readout
$\widetilde\ell_l^{\,+}:=\phi_l(\mathbf{a}_{l+1}^{(k_l^\star)})$ separates each
probe-logit increment into $W_l:=\widetilde\ell_l^{\,+}-\ell_l$ and
$K_l:=\ell_{l+1}-\widetilde\ell_l^{\,+}$.  The normalised probe alignment is
$\psi_l:=\mathbf{v}_l^\top\mathbf{a}_l/\|\mathbf{a}_l\|_2$.  The update
summaries use the change between consecutive normalised residual-update
directions and the excess kurtosis of each block-output activation.
These layer-wise quantities are averaged across bootstrap members before
forming the window statistics below.

The increment features use the same two windows.  A bar denotes the mean over
the indicated window, $\mathrm{TV}$ denotes total variation, and $E_W$ and
$E_K$ denote mean squared magnitudes in the middle window.  The 13 features
are grouped as follows:
\begin{itemize}[nosep,leftmargin=1.5em]
  \item \textbf{Layer level (3):}
  $\bar\psi_{\mathrm{mid}}$, $\bar L_{\mathrm{mid}}$, and
  $L_{m-1}-L_0$.
  \item \textbf{Increments (5):}
  $\bar W_{\mathrm{mid}}$, $\mathrm{TV}(W_{\mathrm{mid}})$,
  $\bar K_{\mathrm{mid}}$, $\mathrm{TV}(K_{\mathrm{mid}})$, and
  $\log(E_W/E_K)$.
  \item \textbf{Updates (2):}
  mean update-direction change and mean activation excess kurtosis in the
  middle window.
  \item \textbf{Late layer (3):}
  $\bar\psi_{\mathrm{late}}$, $\bar K_{\mathrm{late}}$, and
  $\psi_{m-1}$.
\end{itemize}
Table~\ref{tab:source-group-ablation} evaluates the same four groups.

\subsection{Expanded Trajectory}
The expanded trajectory concatenates:
\begin{enumerate}[nosep]
  \item the raw probe-logit trajectory $[L_0, \ldots, L_{m-1}]$ ($m$ dimensions),
  \item the first differences $[\Delta L_0, \ldots, \Delta L_{m-2}]$ ($m{-}1$ dimensions),
  \item six summary statistics: mean, standard deviation, minimum, maximum, final value, and total change ($L_{m-1} - L_0$).
\end{enumerate}
This yields $2m+5$ dimensions.

\section{Auxiliary SNR Analysis and Fisher Ratio}
\label{app:snr-proof}

\subsection{Random-Intercept Model for Probe Logits}

Theorem~\ref{thm:uniform-optimal} uses the full
cross-coordinate covariance.  The following random-intercept model is a
narrower \thistool{} approximation used only to compare the depth
mean with a centred slope statistic.

\begin{remark}[Random-Intercept Model for Probe Logits]
\label{prop:ensemble}
For sample $i$ and layer $l$, the random-intercept idealisation writes
\begin{equation}
  \label{eq:random-intercept}
  L_l^{(i)}
  = \mu_{l,Y_i}+\alpha_i+\xi_{i,l},
  \qquad
  \mu_{l,c}:=\mathbb{E}[L_l\mid Y=c],
\end{equation}
where $\alpha_i$ is a centred sample-level offset shared across depth and
$\xi_{i,l}$ is a centred layer-specific innovation.  The tractable
homoscedastic version assumes
$\operatorname{Var}(\alpha_i)=\sigma_\alpha^2$,
$\operatorname{Var}(\xi_{i,l})=\sigma_\xi^2$, and mutual independence of
$\alpha_i$ and $\{\xi_{i,l}\}_l$ conditional on the class.

The empirical variance decomposition first estimates
$\widehat\alpha_i:=m^{-1}\sum_l(L_l^{(i)}-\mu_{l,Y_i})$ and then forms
$\epsilon_{i,l}:=L_l^{(i)}-\mu_{l,Y_i}-\widehat\alpha_i$, which satisfies
$\sum_l\epsilon_{i,l}=0$ by construction.  This projected residual is used for
the empirical variance accounting, while the innovations $\xi_{i,l}$ define
the auxiliary calculation below.
\end{remark}

\begin{proposition}[Mean Versus the OLS Depth Slope]
\label{prop:snr}
Assume the independent, homoscedastic model in
Equation~\ref{eq:random-intercept}, with $m\geq2$ and $\sigma_\xi^2>0$. Let
$d_l:=\mu_{l,1}-\mu_{l,0}$,
$\bar d:=m^{-1}\sum_l d_l$, and
$\bar l:=(m-1)/2$.  Define
\begin{equation}
  S_{xx}:=\sum_{l=0}^{m-1}(l-\bar l)^2,
  \qquad
  q_l:=\frac{l-\bar l}{S_{xx}},
  \qquad
  \widehat\beta:=\sum_l q_lL_l,
\end{equation}
so that $\widehat\beta$ is the ordinary least-squares depth slope.  For a
scalar statistic $T$, let $\operatorname{SNR}(T)$ denote its absolute
class-conditional mean gap divided by its common within-class standard
deviation under the homoscedastic model.  If $\bar d\neq0$, set
\begin{equation}
  \rho:=\frac{m|\sum_l q_ld_l|}{|\bar d|},
  \qquad
  \gamma:=\sqrt{1+\frac{m\sigma_\alpha^2}{\sigma_\xi^2}}.
\end{equation}
Then
\begin{equation}
  \label{eq:snr-ratio}
  \frac{\operatorname{SNR}(\widehat\beta)}
       {\operatorname{SNR}(m^{-1}\sum_lL_l)}
  =\frac{\rho\gamma}{\sqrt{12}}
   \sqrt{1-\frac{1}{m^2}}.
\end{equation}
Thus a slowly varying class-gap profile makes a centred depth filter weak
relative to the mean in this idealised model.  This proposition concerns one
comparator and one covariance model.  Theorem~\ref{thm:uniform-optimal} gives
the comparison with unrestricted linear weighting.
\end{proposition}

\subsection{Proof of Proposition~\ref{prop:snr}}

\begin{proof}
Under Equation~\ref{eq:random-intercept}, write
\begin{equation}
  L_l=\mu_{l,Y}+\alpha+\xi_l,
\end{equation}
where, conditional on either class, the centred random variables satisfy
\begin{equation}
  \operatorname{Var}(\alpha)=\sigma_\alpha^2,
  \qquad
  \operatorname{Var}(\xi_l)=\sigma_\xi^2,
\end{equation}
and $\alpha,\xi_0,\ldots,\xi_{m-1}$ are mutually independent.  For a
scalar statistic $T$, define its homoscedastic signal-to-noise ratio as the
absolute class-conditional mean gap divided by its common within-class
standard deviation.

The depth mean has class gap $\bar d$ and variance
\begin{equation}
  \operatorname{Var}\!\left(\frac1m\sum_lL_l\,\middle|\,Y=c\right)
  =\sigma_\alpha^2+\frac{\sigma_\xi^2}{m}.
  \label{eq:snr-mean}
\end{equation}
For the OLS weights $q_l=(l-\bar l)/S_{xx}$, we have
$\sum_lq_l=0$ and
\begin{equation}
  \sum_lq_l^2
  =\frac{1}{S_{xx}},
  \qquad
  S_{xx}
  =\sum_{l=0}^{m-1}(l-\bar l)^2
  =\frac{m(m^2-1)}{12}.
\end{equation}
The common offset therefore cancels from
$\widehat\beta=\sum_lq_lL_l$, whose class gap and variance are
\begin{equation}
  \left|\sum_lq_ld_l\right|,
  \qquad
  \operatorname{Var}(\widehat\beta\mid Y=c)
  =\frac{12\sigma_\xi^2}{m(m^2-1)}.
  \label{eq:snr-slope}
\end{equation}
Combining Equations~\ref{eq:snr-mean} to~\ref{eq:snr-slope} gives
\begin{equation}
  \label{eq:snr-ratio-exact}
\begin{aligned}
  \frac{\operatorname{SNR}(\widehat\beta)}
       {\operatorname{SNR}(m^{-1}\sum_lL_l)}
  &=
  \frac{|\sum_lq_ld_l|}{|\bar d|}
  \sqrt{\frac{m(m^2-1)
  (\sigma_\alpha^2+\sigma_\xi^2/m)}
  {12\sigma_\xi^2}} \\
  &=
  \frac{\rho\gamma}{\sqrt{12}}
  \sqrt{1-\frac{1}{m^2}},
\end{aligned}
\end{equation}
because
$\rho=m|\sum_lq_ld_l|/|\bar d|$ and
$\gamma^2=1+m\sigma_\alpha^2/\sigma_\xi^2$.
This is Equation~\ref{eq:snr-ratio}.
\end{proof}

\begin{remark}[Finite-depth factor]
\label{rem:ols-approx-bound}
The simplified expression $\rho\gamma/\sqrt{12}$ omits the multiplicative
factor $\sqrt{1-m^{-2}}$.  For $m\geq2$,
\begin{equation}
  \frac{1}{2m^2}
  <1-\sqrt{1-\frac1{m^2}}
  <\frac{1}{m^2},
\end{equation}
which follows from
$1-\sqrt{1-x}=x/(1+\sqrt{1-x})$ with $x=m^{-2}$.
Thus the omitted factor and its direction are explicit.
\end{remark}

\paragraph{Zero-sum linear filters.}
The same working model gives a useful extension of
Proposition~\ref{prop:snr}.  Write
$\mathbf{d}=\bar d\,\mathbf{1}+\mathbf{r}$, where
$\mathbf{1}^\top\mathbf{r}=0$, and define
\begin{equation}
  \label{eq:rho-shape}
  \rho_{\perp}
  :=\frac{\|\mathbf{r}\|_2}{|\bar d|\sqrt m}.
\end{equation}
For every unit vector $\mathbf{q}$ satisfying
$\mathbf{1}^\top\mathbf{q}=0$, the common offset cancels and
Cauchy-Schwarz gives
\begin{equation}
  |\mathbf{q}^\top\mathbf{d}|
  =|\mathbf{q}^\top\mathbf{r}|
  \leq\|\mathbf{r}\|_2.
\end{equation}
Since
$\operatorname{Var}(\mathbf{q}^\top\boldsymbol{\tau}\mid Y=c)
=\sigma_\xi^2$, it follows that
\begin{equation}
  \sup_{\substack{\mathbf{q}\neq\mathbf{0}\\
                   \mathbf{1}^\top\mathbf{q}=0}}
  \frac{\operatorname{SNR}(\mathbf{q}^\top\boldsymbol{\tau})}
       {\operatorname{SNR}(m^{-1}\mathbf{1}^\top\boldsymbol{\tau})}
  \leq\gamma\rho_{\perp}.
\end{equation}
The ratio is scale-invariant, so restricting the supremum to unit vectors
loses nothing.  If $\mathbf{r}\neq\mathbf{0}$, equality is attained by
$\mathbf{q}=\mathbf{r}/\|\mathbf{r}\|_2$.

\subsection{Proof of Theorem~\ref{thm:uniform-optimal}}
\label{app:uniform-proof}

\begin{proof}
We take $I=\mathcal{K}_m$ and follow the definitions in
\S~\ref{sec:detector}.  The argument has four parts.

\paragraph{1. Fisher-optimal direction.}
Because $\Sigma\succ0$, the bilinear form
$\langle\mathbf{x},\mathbf{y}\rangle_\Sigma
:=\mathbf{x}^\top\Sigma\mathbf{y}$ is an inner product.  With
$\mathbf{w}_\star=\Sigma^{-1}\mathbf{d}$, we have
$\mathbf{d}=\Sigma\mathbf{w}_\star$, and therefore, for every nonzero
$\mathbf{w}$,
\begin{equation}
  (\mathbf{w}^\top\mathbf{d})^2
  =\langle\mathbf{w},\mathbf{w}_\star\rangle_\Sigma^2
  \leq
  (\mathbf{w}^\top\Sigma\mathbf{w})
  (\mathbf{w}_\star^\top\Sigma\mathbf{w}_\star).
\end{equation}
Dividing by the positive first factor yields
\begin{equation}
  J(\mathbf{w})
  \leq
  \mathbf{w}_\star^\top\Sigma\mathbf{w}_\star
  =\mathbf{d}^\top\Sigma^{-1}\mathbf{d}
  =J(\mathbf{w}_\star).
\end{equation}
Since $\mathbf{d}\neq\mathbf{0}$, positive definiteness also gives
$\mathbf{w}_\star\neq\mathbf{0}$ and
\begin{equation}
  \mathbf{w}_\star^\top\mathbf{d}
  =\mathbf{d}^\top\Sigma^{-1}\mathbf{d}>0.
\end{equation}
Thus $\mathbf{w}_\star$ is the oriented Fisher-optimal direction.  Every
nonzero rescaling attains the same Fisher criterion, while positive rescaling
preserves its class orientation.

\paragraph{2. Orthogonal decomposition around the uniform direction.}
For the full support, the vector
$\mathbf{u}_I=m^{-1/2}\mathbf{1}$ has unit norm and can be completed
to an orthogonal matrix $U_I=[\mathbf{u}_I,Q_I]$.  In the coordinates of
Equation~\ref{eq:uniform-blocks}, positive definiteness of $\Sigma$ gives
$a>0$, $C\succ0$, and the positive Schur complement
\begin{equation}
  \widetilde a
  :=a-\mathbf{b}^\top C^{-1}\mathbf{b}>0.
  \label{eq:positive-schur}
\end{equation}
Consequently
\begin{equation}
  0\leq
  \theta_I=\frac{\mathbf{b}^\top C^{-1}\mathbf{b}}{a}
  <1.
\end{equation}
The quantities are independent of the chosen basis of the orthogonal complement.
Indeed, replacing $Q_I$ by $Q_IR$, with $R$ orthogonal, maps
\begin{equation}
  \mathbf{z}\mapsto R^\top\mathbf{z},
  \qquad
  \mathbf{b}\mapsto R^\top\mathbf{b},
  \qquad
  C\mapsto R^\top CR,
\end{equation}
and leaves each quadratic or bilinear scalar in
Equation~\ref{eq:fisher-diagnostics} unchanged.

\paragraph{3. Exact Fisher ratio.}
Let
\begin{equation}
  M:=U_I^\top\Sigma U_I
  =
  \begin{pmatrix}
    a&\mathbf{b}^\top\\
    \mathbf{b}&C
  \end{pmatrix}.
\end{equation}
The block inverse associated with Equation~\ref{eq:positive-schur} is
\begin{equation}
  M^{-1}
  =
  \begin{pmatrix}
    \widetilde a^{-1}
    &-\widetilde a^{-1}\mathbf{b}^\top C^{-1}\\
    -C^{-1}\mathbf{b}\widetilde a^{-1}
    &C^{-1}
      +C^{-1}\mathbf{b}\widetilde a^{-1}\mathbf{b}^\top C^{-1}
  \end{pmatrix}.
\end{equation}
Hence
\begin{equation}
  \label{eq:block-fisher-value}
\begin{aligned}
  J(\mathbf{w}_\star)
  &=\mathbf{d}^\top\Sigma^{-1}\mathbf{d}\\
  &=
  \begin{pmatrix}s\\\mathbf{z}\end{pmatrix}^{\!\top}
  M^{-1}
  \begin{pmatrix}s\\\mathbf{z}\end{pmatrix}\\
  &=
  \frac{(s-\mathbf{b}^\top C^{-1}\mathbf{z})^2}
       {a-\mathbf{b}^\top C^{-1}\mathbf{b}}
  +\mathbf{z}^\top C^{-1}\mathbf{z}.
\end{aligned}
\end{equation}
The uniform direction has
\begin{equation}
  J(\mathbf{u}_I)
  =\frac{s^2}{a}>0.
\end{equation}
Dividing Equation~\ref{eq:block-fisher-value} by $s^2/a$ gives
\begin{equation}
\begin{aligned}
  R_I^2
  &=
  \frac{a}{s^2}
  \left[
    \frac{(s-\mathbf{b}^\top C^{-1}\mathbf{z})^2}
         {a-\mathbf{b}^\top C^{-1}\mathbf{b}}
    +\mathbf{z}^\top C^{-1}\mathbf{z}
  \right]\\
  &=
  \chi_I^2+\frac{(1-t_I)^2}{1-\theta_I},
\end{aligned}
\end{equation}
which is the additive identity in Equation~\ref{eq:fisher-gap-exact}.  Part~1
gives $R_I^2\geq1$.

\paragraph{4. Bounds for the cross term.}
Cauchy-Schwarz in the $C^{-1}$-inner product gives
\begin{equation}
  (\mathbf{b}^\top C^{-1}\mathbf{z})^2
  \leq
  (\mathbf{b}^\top C^{-1}\mathbf{b})
  (\mathbf{z}^\top C^{-1}\mathbf{z}).
\end{equation}
After division by $s^2$, this becomes
\begin{equation}
  t_I^2\leq\theta_I\chi_I^2,
  \qquad
  |t_I|\leq\sqrt{\theta_I}\,\chi_I.
\end{equation}
\end{proof}

\begin{corollary}[Weak Bound on the Fisher Ratio]
\label{cor:weak-bound}
Under the assumptions of Theorem~\ref{thm:uniform-optimal}, if
$\varepsilon\geq0$, $\eta\geq0$, $\chi_I\leq\varepsilon$, and
$\theta_I\leq\eta^2<1$, then
\begin{equation}
  \label{eq:fisher-weak-bound}
  1\leq R_I^2\leq B(\varepsilon,\eta),
  \qquad
  B(\varepsilon,\eta)
  :=\varepsilon^2+
    \frac{(1+\eta\varepsilon)^2}{1-\eta^2}.
\end{equation}
If the additional assumptions of Corollary~\ref{cor:gaussian-auroc} hold,
the AUROC difference in Equation~\ref{eq:auroc-gap} is at most
$(\sqrt{B(\varepsilon,\eta)}-1)\Delta_I/(2\sqrt\pi)$.
\end{corollary}

\begin{proof}
The constraint $|t_I|\leq\sqrt{\theta_I}\,\chi_I$ gives
$|t_I|\leq\eta\varepsilon$, and the exact identity in
Equation~\ref{eq:fisher-gap-exact} yields
\begin{equation}
\begin{aligned}
  R_I^2
  &\leq
  \varepsilon^2+
  \frac{(1+|t_I|)^2}{1-\theta_I}\\
  &\leq
  \varepsilon^2+
  \frac{(1+\eta\varepsilon)^2}{1-\eta^2}
  =B(\varepsilon,\eta).
\end{aligned}
\end{equation}
Together with $R_I^2\geq1$ from Theorem~\ref{thm:uniform-optimal}, this
proves Equation~\ref{eq:fisher-weak-bound}. Under the additional assumptions of
Corollary~\ref{cor:gaussian-auroc}, the AUROC statement follows from
Equation~\ref{eq:auroc-gap} and $R_I\leq\sqrt{B(\varepsilon,\eta)}$.
\end{proof}

\begin{remark}[Limiting Cases of the Fisher Decomposition]
If the class gap lies in the uniform direction and that direction is
covariance-decoupled from its orthogonal complement, then
$\chi_I=\theta_I=t_I=0$ and Equation~\ref{eq:fisher-gap-exact} gives $R_I=1$.
Away from this case, the same equation records the orthogonal mean difference
and cross-covariance as separate terms in the ratio.
\end{remark}

\subsection{Proof of Corollary~\ref{cor:gaussian-auroc}}
\label{app:gaussian-auroc-proof}

\begin{proof}
Let $\mathbf{w}\in\{\mathbf{u}_I,\mathbf{w}_\star\}$ satisfy the Gaussian
difference condition.  Its nondegenerate Gaussian difference has no atoms,
so the tie probability is zero.  If
$\mathbf{w}^\top\mathbf{d}>0$, standardising
$D_{\mathbf{w}}$ gives
\begin{equation}
\begin{aligned}
  \operatorname{AUROC}(\mathbf{w})
  &=\mathbb{P}(D_{\mathbf{w}}>0)\\
  &=\Phi\!\left(
      \frac{\mathbf{w}^\top\mathbf{d}}
           {\sqrt{2\mathbf{w}^\top\Sigma\mathbf{w}}}
    \right)\\
  &=\Phi\!\left(\sqrt{\frac{J(\mathbf{w})}{2}}\right).
\end{aligned}
\end{equation}
The assumption $s>0$ gives the required sign for $\mathbf{u}_I$, while
$\mathbf{w}_\star^\top\mathbf{d}=\mathbf{d}^\top\Sigma^{-1}\mathbf{d}>0$.  Since
$\sqrt{J(\mathbf{w}_\star)}=R_I\Delta_I$,
\begin{equation}
  \operatorname{AUROC}(\mathbf{w}_\star)
  =\Phi\!\left(\frac{R_I\Delta_I}{\sqrt2}\right),
  \qquad
  \operatorname{AUROC}(\mathbf{u}_I)
  =\Phi\!\left(\frac{\Delta_I}{\sqrt2}\right).
\end{equation}
Subtracting the two identities gives
Equation~\ref{eq:auroc-gap-exact}.
The monotonicity of $\Phi$, together with $R_I\geq1$, gives the lower bound in
Equation~\ref{eq:auroc-gap}.  The standard normal density is bounded above by
$1/\sqrt{2\pi}$, so $\Phi$ is $1/\sqrt{2\pi}$-Lipschitz.  Therefore
\begin{equation}
\begin{aligned}
  \operatorname{AUROC}(\mathbf{w}_\star)
  -\operatorname{AUROC}(\mathbf{u}_I)
  &\leq
  \frac{1}{\sqrt{2\pi}}
  \frac{(R_I-1)\Delta_I}{\sqrt2}\\
  &=\frac{(R_I-1)\Delta_I}{2\sqrt\pi}.
\end{aligned}
\end{equation}
\end{proof}

\subsection{Proof of Corollary~\ref{cor:approximate-auroc}}
\label{app:approximate-auroc-proof}

\begin{proof}
Write
\begin{equation}
  A_\star:=\operatorname{AUROC}(\mathbf{w}_\star),
  \qquad
  A_I:=\operatorname{AUROC}(\mathbf{u}_I),
\end{equation}
and
\begin{equation}
  P_\star:=\Phi\!\left(\frac{R_I\Delta_I}{\sqrt2}\right),
  \qquad
  P_I:=\Phi\!\left(\frac{\Delta_I}{\sqrt2}\right).
\end{equation}
The two approximation assumptions imply
\begin{equation}
  A_\star-P_\star\leq\kappa_\star,
  \qquad
  P_I-A_I\leq\kappa_I.
\end{equation}
Because Theorem~\ref{thm:uniform-optimal} gives $R_I\geq1$ and
$\Delta_I\geq0$, the standard-normal Lipschitz bound used above gives
\begin{equation}
  P_\star-P_I
  \leq\frac{(R_I-1)\Delta_I}{2\sqrt\pi}.
\end{equation}
Therefore
\begin{equation}
\begin{aligned}
  A_\star-A_I
  &=(A_\star-P_\star)+(P_\star-P_I)+(P_I-A_I)\\
  &\leq
    \frac{(R_I-1)\Delta_I}{2\sqrt\pi}
    +\kappa_\star+\kappa_I,
\end{aligned}
\end{equation}
which is Equation~\ref{eq:auroc-gap-approximate}.
\end{proof}

\begin{remark}[Scope of the Gaussian assumption]
Theorem~\ref{thm:uniform-optimal} uses only class-conditional second moments
and positive definiteness of $\Sigma$.  Corollary~\ref{cor:gaussian-auroc}
adds a one-dimensional Gaussian difference assumption for the two displayed
weights.  Independent Gaussian class projections are sufficient.  The result
also allows a non-Gaussian joint score vector, unequal class covariances, and
dependent coordinates.  Corollary~\ref{cor:approximate-auroc} replaces
the exact Gaussian law by the two explicit AUROC approximation errors.
\end{remark}

\subsection{Empirical Diagnostics}
\label{app:snr-empirical}

The auxiliary random-intercept expression in Proposition~\ref{prop:snr}
separates the depth-gap ratio $\rho$ from the offset penalty
$\gamma$.  Observed trajectories can contain dependent innovations, so we
also report a covariance plug-in slope-to-mean ratio using the
sample covariance of the projected residual trajectories.  This auxiliary
quantity concerns Proposition~\ref{prop:snr} and is distinct from the
equal-class pooled covariance used by Theorem~\ref{thm:uniform-optimal}.
The displayed values are descriptive full-sample diagnostics.  The
held-out plug-in Fisher comparison appears in Table~\ref{tab:single-layer-gap}.

\begin{table*}[t]
\centering
\caption{SNR measurements across six models and three dataset families.
$\rho$: depth-gap ratio; $\gamma$: offset penalty;
$R_{\text{plugin}}$: covariance-aware predicted ratio $\mathrm{SNR}(\hat{\beta})/\mathrm{SNR}(\bar{L})$;
$R_{\text{emp}}$: empirically measured ratio.  The final column reports
$R_{\text{emp}}/R_{\text{plugin}}$.}
\label{tab:snr-measurements}
\small
\setlength{\tabcolsep}{5pt}
\begin{tabular}{llccccc}
\toprule
\textbf{Model} & \textbf{Dataset} & $\boldsymbol{\rho}$ & $\boldsymbol{\gamma}$ & $\boldsymbol{R_{\text{plugin}}}$ & $\boldsymbol{R_{\text{emp}}}$ & $\boldsymbol{R_{\text{emp}}/R_{\text{plugin}}}$ \\
\midrule
\multirow{3}{*}{Mistral-Small-24B}
  & HaluEval2  & 0.666 & 8.50 & 0.568 & 0.567 & 0.999 \\
  & TrueFalse & 1.212 & 6.07 & 0.823 & 0.824 & 1.000 \\
  & HELM      & 0.197 & 6.83 & 0.188 & 0.188 & 1.002 \\
\midrule
\multirow{3}{*}{Qwen2.5-32B}
  & HaluEval2  & 0.945 & 8.67 & 0.700 & 0.706 & 1.008 \\
  & TrueFalse & 1.777 & 6.25 & 1.017 & 1.017 & 1.000 \\
  & HELM      & 0.144 & 7.90 & 0.148 & 0.148 & 0.999 \\
\midrule
\multirow{3}{*}{Qwen2.5-14B}
  & HaluEval2  & 0.902 & 8.23 & 0.669 & 0.672 & 1.005 \\
  & TrueFalse & 1.882 & 5.87 & 1.012 & 1.013 & 1.000 \\
  & HELM      & 0.203 & 7.79 & 0.205 & 0.206 & 1.004 \\
\midrule
\multirow{3}{*}{Qwen2.5-7B}
  & HaluEval2  & 0.730 & 6.70 & 0.568 & 0.572 & 1.006 \\
  & TrueFalse & 1.982 & 4.41 & 1.006 & 1.006 & 1.000 \\
  & HELM      & 0.105 & 6.16 & 0.102 & 0.103 & 1.001 \\
\midrule
\multirow{3}{*}{Llama-3.1-8B}
  & HaluEval2  & 0.525 & 6.86 & 0.487 & 0.486 & 0.997 \\
  & TrueFalse & 1.089 & 5.53 & 0.827 & 0.827 & 1.000 \\
  & HELM      & 0.125 & 6.68 & 0.146 & 0.146 & 1.002 \\
\midrule
\multirow{3}{*}{Llama-2-7B}
  & HaluEval2  & 0.499 & 7.02 & 0.507 & 0.509 & 1.003 \\
  & TrueFalse & 1.368 & 5.21 & 0.846 & 0.847 & 1.001 \\
  & HELM      & 0.096 & 6.23 & 0.113 & 0.113 & 1.002 \\
\bottomrule
\end{tabular}
\end{table*}

\subsection{Estimation Cost of Fisher-Optimal Aggregation}
\label{rem:why-uniform}

Fisher weighting depends on estimated class means and cross-layer covariance.
Covariance inversion can amplify estimation error along directions with small
eigenvalues, motivating regularisation~\citep{ledoit2004well}.  For $m$ layers,
the pooled covariance contains $m(m+1)/2$ distinct parameters, or $528$ when
$m=32$.  Greater depth increases this number, while limited minority-class
data constrain covariance estimation.

We examine whether the gain characterised by Theorem~\ref{thm:uniform-optimal}
translates into better held-out prediction by comparing the training-fold
diagnostics and held-out performance across training sizes.  The depth mean
uses the same fitted probes with fixed aggregation weights.  Once fitted,
both aggregation rules require $\mathcal{O}(m)$ work per example.

\paragraph{Fitting and evaluation.}
On each outer-training fold, we estimate the two class means and covariances
separately, form the equal-class pooled covariance
$\widehat\Sigma=(\widehat\Sigma_0+\widehat\Sigma_1)/2$, and use
\begin{equation}
  \widehat\Sigma_{\mathrm{reg}}
  :=\widehat\Sigma+10^{-6}\frac{\operatorname{tr}(\widehat\Sigma)}{m}\mathbf I_m,
  \qquad
  \widehat{\mathbf{w}}_\star
  :=\widehat\Sigma_{\mathrm{reg}}^{-1}\widehat{\mathbf{d}}.
\end{equation}
The fixed ridge stabilises inversion and is not tuned on AUROC.
After the weight vector is
fixed, it is oriented from the training fold so that larger values point toward
class~$1$.  Let $\widehat\mu_{c,\mathbf{w}}$ and
$\widehat\sigma^2_{c,\mathbf{w}}$ denote the mean and variance of its scalar
scores on the held-out fold.  When their summed variance is positive, we
compute
\begin{equation}
  \label{eq:empirical-pairwise-snr}
  \widehat\Gamma(\mathbf{w})
  :=\frac{\widehat\mu_{1,\mathbf{w}}
           -\widehat\mu_{0,\mathbf{w}}}
          {\sqrt{\widehat\sigma^2_{1,\mathbf{w}}
                 +\widehat\sigma^2_{0,\mathbf{w}}}},
  \qquad
  \widehat\kappa(\mathbf{w})
  :=\left|
      \widehat{\operatorname{AUROC}}(\mathbf{w})
      -\Phi(\widehat\Gamma(\mathbf{w}))
    \right|.
\end{equation}
The first quantity is the signed pairwise SNR on held-out scores, and the
second is the Gaussian AUROC approximation error.
Table~\ref{tab:fisher-diagnostics} reports $\widehat R_I$, the held-out SNR ratio,
$\widehat\kappa$, and the decomposition estimates $\widehat\theta_I$,
$\widehat\chi_I$, and $\widehat t_I$ on the 18 conditions of
Table~\ref{tab:single-layer-gap},
together with $m$ and the ratio of covariance parameters to
minority-class training examples.  $\widehat R_I$ exceeds the held-out SNR ratio in
every condition.

\begin{table}[t]
\centering
\caption{\textbf{Fisher-ratio diagnostics for Table~\ref{tab:single-layer-gap}.}
$m$: number of layers; $N_{\min}$: minority-class training count.
$m(m{+}1)/(2N_{\min})$ measures covariance parameters per minority-class example.
$\widehat R_I$, $\widehat\theta_I$, $\widehat\chi_I$, and $\widehat t_I$ are
training-fold diagnostics in Equation~\ref{eq:fisher-gap-exact}.
The held-out SNR ratio is Fisher SNR divided by depth-mean SNR.
$\widehat\kappa$ is the held-out Gaussian approximation error in AUROC points
(Equation~\ref{eq:empirical-pairwise-snr}).
Values are five-fold means; $\pm$ denotes standard deviation.}
\label{tab:fisher-diagnostics}
\scriptsize
\setlength{\tabcolsep}{3.5pt}
\begin{tabular}{@{}ll r r c r rr rrr@{}}
\toprule
Model & Dataset & $m$ & $\frac{m(m+1)}{2N_{\min}}$ & $\widehat R_I$ & Held-out SNR ratio & $\widehat\kappa_{\mathrm{mean}}$ & $\widehat\kappa_{\mathrm{Fisher}}$ & $\widehat\theta_I$ & $\widehat\chi_I$ & $\widehat t_I$ \\
\midrule
\multirow{3}{*}{Llama-2-7B}
 & HaluEval2 & 32 & 0.41 & $1.024_{\pm0.004}$ & 0.986 & 0.25 & 0.36 & 0.479 & 0.696 & 0.458 \\
 & TrueFalse & 32 & 0.21 & $1.048_{\pm0.006}$ & 1.028 & 0.83 & 1.19 & 0.944 & 0.998 & 0.924 \\
 & HELM & 32 & 0.44 & $1.044_{\pm0.004}$ & 1.014 & 0.33 & 0.46 & 0.357 & 0.573 & 0.301 \\
\addlinespace[1.5pt]
\multirow{3}{*}{Llama-3.1-8B}
 & HaluEval2 & 32 & 0.41 & $1.043_{\pm0.001}$ & 1.011 & 0.55 & 0.59 & 0.469 & 0.719 & 0.450 \\
 & TrueFalse & 32 & 0.21 & $1.050_{\pm0.002}$ & 1.036 & 1.12 & 1.51 & 0.957 & 1.021 & 0.950 \\
 & HELM & 32 & 0.44 & $1.040_{\pm0.004}$ & 1.001 & 0.44 & 0.60 & 0.366 & 0.611 & 0.329 \\
\addlinespace[1.5pt]
\multirow{3}{*}{Qwen2.5-7B}
 & HaluEval2 & 28 & 0.32 & $1.039_{\pm0.006}$ & 1.028 & 0.55 & 0.65 & 0.340 & 0.687 & 0.368 \\
 & TrueFalse & 28 & 0.16 & $1.103_{\pm0.004}$ & 1.082 & 1.07 & 1.59 & 0.929 & 1.071 & 0.930 \\
 & HELM & 28 & 0.34 & $1.038_{\pm0.005}$ & 1.001 & 0.23 & 0.41 & 0.311 & 0.558 & 0.273 \\
\addlinespace[1.5pt]
\multirow{3}{*}{Qwen2.5-14B}
 & HaluEval2 & 48 & 0.92 & $1.058_{\pm0.005}$ & 0.990 & 0.67 & 0.69 & 0.556 & 0.861 & 0.591 \\
 & TrueFalse & 48 & 0.47 & $1.089_{\pm0.008}$ & 1.070 & 1.01 & 1.36 & 0.949 & 1.052 & 0.936 \\
 & HELM & 48 & 0.98 & $1.051_{\pm0.004}$ & 1.004 & 0.37 & 0.65 & 0.507 & 0.729 & 0.469 \\
\addlinespace[1.5pt]
\multirow{3}{*}{Qwen2.5-32B}
 & HaluEval2 & 64 & 1.62 & $1.071_{\pm0.010}$ & 1.012 & 0.71 & 0.79 & 0.646 & 0.883 & 0.639 \\
 & TrueFalse & 64 & 0.83 & $1.118_{\pm0.004}$ & 1.091 & 0.91 & 1.22 & 0.956 & 1.085 & 0.943 \\
 & HELM & 64 & 1.74 & $1.062_{\pm0.006}$ & 0.993 & 0.56 & 0.58 & 0.498 & 0.715 & 0.443 \\
\addlinespace[1.5pt]
\multirow{3}{*}{Mistral-Small-24B}
 & HaluEval2 & 40 & 0.64 & $1.039_{\pm0.005}$ & 0.967 & 0.80 & 0.72 & 0.559 & 0.803 & 0.562 \\
 & TrueFalse & 40 & 0.33 & $1.053_{\pm0.002}$ & 1.037 & 1.22 & 1.45 & 0.941 & 1.027 & 0.944 \\
 & HELM & 40 & 0.69 & $1.049_{\pm0.004}$ & 1.003 & 0.51 & 0.81 & 0.489 & 0.765 & 0.487 \\
\bottomrule
\end{tabular}
\vspace{-3mm}
\end{table}

\paragraph{Decomposition diagnostics.}
All training-fold decomposition diagnostics use the same
$\widehat\Sigma_{\mathrm{reg}}$ as the fitted Fisher direction, so the identity
$\widehat R_I^2=\widehat\chi_I^2+(1-\widehat t_I)^2/(1-\widehat\theta_I)$
holds exactly within each fold.  Table~\ref{tab:fisher-diagnostics} reports
rounded fold means, which need not preserve this nonlinear identity exactly.
$\widehat t_I$ reaches between 87 and 96 percent of its Cauchy-Schwarz bound
$\sqrt{\widehat\theta_I}\,\widehat\chi_I$, so the class
separation carried by the orthogonal complement is largely redundant with the
uniform direction.  Within each fold, setting $\widehat t_I=0$ while retaining
the measured $\widehat\theta_I$ and $\widehat\chi_I$ gives condition-level
fold-averaged ratios between 1.33 and 4.93, largest on TrueFalse where
$\widehat\theta_I$ reaches 0.96.  This coupling keeps the measured
$\widehat R_I$ close to one while the orthogonal complement still carries
class separation.
In the decoupled limit of the remark following the proof of
Theorem~\ref{thm:uniform-optimal}, all three quantities vanish, and the observed
values place the trajectories away from that limit, while the uniform
direction still retains at least 89 percent of the Fisher-optimal pooled SNR
at the condition-level training-fold means.

\paragraph{Covariance-shrinkage path.}
\label{app:shrinkage-path}
The shrinkage path retains the same fixed ridge and uses
$\widehat\Sigma_\alpha=(1-\alpha)\widehat\Sigma+(\alpha+10^{-6})\,\operatorname{tr}(\widehat\Sigma)/m\cdot\mathbf I_m$
in place of $\widehat\Sigma_{\mathrm{reg}}$ in $\widehat{\mathbf{w}}_\star$ for
$\alpha\in\{0,0.05,\ldots,1\}$, with the probe logits and outer folds of
Table~\ref{tab:single-layer-gap}.  The path starts at the plug-in Fisher score
($\alpha=0$) and ends at the mean-gap direction
$\widehat{\mathbf{d}}$ ($\alpha=1$).  The same path is evaluated at each
training fraction in Appendix~\ref{app:nmin-sweep}.
Table~\ref{tab:shrinkage-path} reports the largest held-out AUROC gain over
the depth mean and its selected $\alpha$ at full training size and at one
sixteenth, together with the unshrunk $\alpha=0$ difference at one sixteenth.
Across all tested fractions, the largest gain is 0.79 points, and at full
training size the gain averages 0.24 points over the 18 conditions.  At one
sixteenth, the selected $\alpha$ lies between 0.55 and 1.00 and the largest
condition-level gain is 0.58 points.

\begin{table}[H]
\centering
\caption{\textbf{Covariance shrinkage across training sizes.}
AUROC of the shrunk Fisher score minus the depth mean, in AUROC points.
For each condition and fraction, $\alpha^\star$ maximises the mean difference
over $\alpha\in\{0,0.05,\ldots,1\}$ on the held-out folds.  The final column
reports the unshrunk score at one sixteenth of the training fold.}
\label{tab:shrinkage-path}
\scriptsize
\setlength{\tabcolsep}{5.5pt}
\begin{tabular}{@{}ll rr@{\hspace{8pt}}rrr@{}}
\toprule
& & \multicolumn{2}{c}{Full training fold}
& \multicolumn{3}{c}{One sixteenth} \\
\cmidrule(lr){3-4}\cmidrule(l){5-7}
Model & Dataset & Max. gain & $\alpha^\star$ & Max. gain & $\alpha^\star$ & $\alpha=0$ \\
\midrule
\multirow{3}{*}{Llama-2-7B}
 & HaluEval2 & $+0.08$ & 0.80 & $+0.08$ & 0.90 & $-4.70$ \\
 & TrueFalse & $+0.18$ & 0.30 & $+0.14$ & 0.55 & $-0.88$ \\
 & HELM & $+0.34$ & 0.25 & $+0.02$ & 0.75 & $-4.06$ \\
\addlinespace[2pt]
\multirow{3}{*}{Llama-3.1-8B}
 & HaluEval2 & $+0.44$ & 0.15 & $+0.15$ & 0.85 & $-4.01$ \\
 & TrueFalse & $+0.04$ & 1.00 & $+0.04$ & 1.00 & $-0.49$ \\
 & HELM & $+0.09$ & 0.40 & $+0.01$ & 0.95 & $-4.56$ \\
\addlinespace[2pt]
\multirow{3}{*}{Qwen2.5-7B}
 & HaluEval2 & $+0.76$ & 0.15 & $+0.49$ & 0.70 & $-2.52$ \\
 & TrueFalse & $+0.19$ & 0.85 & $+0.19$ & 0.85 & $-0.16$ \\
 & HELM & $+0.08$ & 0.50 & $+0.02$ & 0.90 & $-3.73$ \\
\addlinespace[2pt]
\multirow{3}{*}{Qwen2.5-14B}
 & HaluEval2 & $+0.44$ & 0.75 & $+0.38$ & 0.90 & $-6.43$ \\
 & TrueFalse & $+0.04$ & 1.00 & $+0.04$ & 1.00 & $-0.52$ \\
 & HELM & $+0.25$ & 0.30 & $+0.05$ & 0.90 & $-5.82$ \\
\addlinespace[2pt]
\multirow{3}{*}{Qwen2.5-32B}
 & HaluEval2 & $+0.79$ & 0.50 & $+0.58$ & 0.85 & $-8.50$ \\
 & TrueFalse & $+0.06$ & 0.95 & $+0.06$ & 0.90 & $-0.56$ \\
 & HELM & $+0.33$ & 0.30 & $+0.08$ & 0.85 & $-9.23$ \\
\addlinespace[2pt]
\multirow{3}{*}{Mistral-Small-24B}
 & HaluEval2 & $+0.09$ & 0.90 & $+0.08$ & 1.00 & $-5.40$ \\
 & TrueFalse & $+0.04$ & 0.95 & $+0.04$ & 0.95 & $-0.36$ \\
 & HELM & $+0.08$ & 0.50 & $+0.01$ & 0.95 & $-5.16$ \\
\bottomrule
\end{tabular}
\end{table}

\begin{table}[t]
\centering
\caption{\textbf{Probe-logit scale control.}
Five-fold mean AUROC in percent.  Raw is the deployed depth mean.
Standardised first uses the training-fold mean and standard deviation of each
layer-wise logit.  Probability averages the sigmoid-transformed logits.}
\label{tab:score-coordinate-scale}
\scriptsize
\setlength{\tabcolsep}{3.2pt}
\begin{tabular}{@{}l rrr rrr rrr@{}}
\toprule
& \multicolumn{3}{c}{HaluEval2}
& \multicolumn{3}{c}{TrueFalse}
& \multicolumn{3}{c}{HELM} \\
\cmidrule(lr){2-4}\cmidrule(lr){5-7}\cmidrule(l){8-10}
Model & Raw & Standardised & Probability
& Raw & Standardised & Probability
& Raw & Standardised & Probability \\
\midrule
Llama-2-7B
& \best{80.85} & 80.75 & 80.84
& \best{95.19} & 94.43 & 95.05
& 87.70 & \best{87.71} & 87.64 \\
Llama-3.1-8B
& \best{83.60} & 83.42 & 83.53
& \best{97.30} & 96.96 & 97.15
& 88.80 & 88.79 & \best{88.83} \\
Qwen2.5-7B
& \best{82.03} & 81.70 & 81.92
& \best{97.43} & 96.29 & 97.07
& \best{87.85} & 87.69 & 87.66 \\
Qwen2.5-14B
& \best{83.56} & 83.18 & 83.53
& \best{98.34} & 97.89 & 98.13
& 88.34 & 88.40 & \best{88.44} \\
Qwen2.5-32B
& \best{84.03} & 83.58 & 83.94
& \best{98.51} & 98.07 & 98.31
& 88.51 & \best{88.69} & 88.66 \\
Mistral-Small-24B
& 84.71 & 84.59 & \best{84.74}
& \best{98.22} & 98.03 & 98.08
& \best{89.50} & 89.49 & 89.41 \\
\bottomrule
\end{tabular}%
\vspace{-3mm}
\end{table}

\paragraph{Training-size sweep.}
\label{app:nmin-sweep}
For each outer fold, the probes are fitted on the full training fold and the
trajectories are held fixed.  The training trajectories are subsampled per
class without replacement at fractions $1,\tfrac12,\tfrac14,\tfrac18,\tfrac1{16}$,
with five draws per fraction averaged within the fold and then across the
five folds.  Only $\widehat{\mathbf{d}}$ and $\widehat\Sigma$ are re-estimated,
with the ridge of Table~\ref{tab:single-layer-gap} and no shrinkage, and the
depth mean is unchanged.  Figure~\ref{fig:nmin-sweep} reports the held-out
AUROC difference and the held-out SNR ratio against the parameter ratio
$m(m+1)/(2N_{\min})$.  The difference is non-increasing along the sweep in
all 18 conditions, and the held-out ratio decreases and the training-fold
$\widehat R_I$ increases in every condition.  At full training size, the
estimated coupling saturation
$|\widehat t_I|/(\sqrt{\widehat\theta_I}\,\widehat\chi_I)$
lies between 0.87 and 0.96.  It decreases after every successive subsampling
step in all 18 conditions, so subsampling moves the plug-in decomposition away
from the Cauchy-Schwarz bound.

\subsection{Location Estimator Ablation: Mean vs.\ Median}
\label{app:mean-vs-median}

Theorem~\ref{thm:uniform-optimal} concerns linear scores.  A complementary
question compares the depth-wise median with the depth mean under possible outlying
layers.

\paragraph{Gaussian reference.}
Consider the centred scores
$X_l:=L_l-\mu_{l,Y}=\alpha+\xi_l$ in
Equation~\ref{eq:random-intercept}.  Conditional on either class, assume
$\xi_0,\ldots,\xi_{m-1}$ are iid $\mathcal{N}(0,\sigma_\xi^2)$ with
$\sigma_\xi^2>0$ and independent of $\alpha$.
Let $\bar X$ and $\widetilde X$ denote the mean and median of these centred scores.
Both estimate the common residual location $\alpha$, with
\begin{equation}
  \operatorname{Var}(\bar X\mid\alpha,Y)=\frac{\sigma_\xi^2}{m},
  \qquad
  \operatorname{Var}(\widetilde X\mid\alpha,Y)
  \sim\frac{\pi\sigma_\xi^2}{2m}
  \quad (m\to\infty).
\end{equation}
The median therefore has asymptotic relative efficiency $2/\pi\approx0.637$
for estimating this location under Gaussian residuals.  The empirical
ablation below tests whether replacing the mean with the median improves
ranking on the original probe logits.

\paragraph{Empirical comparison.}
Across three models and three benchmarks, the median and 10\%-trimmed mean
remain within $0.3$ AUROC points of the depth mean.  The depth mean matches or
exceeds the median in all nine conditions, and their score correlation exceeds
$0.99$ in eight conditions ($0.989$ in the remaining condition).

\subsection{Trajectory and Increment-Feature Controls}
\label{app:feature-controls}

Table~\ref{tab:source-group-ablation} separates the 13 increment-based
statistics into four groups.  The layer-level group lies within $0.5$
percentage points of the full representation in the three displayed
conditions, while the update-distribution group has lower AUROC in
isolation.

\begin{table}[t]
\centering
\caption{\textbf{Increment-feature group ablation} on Llama-3.1-8B
(AUROC $\times100$, averaged over five runs).  Full uses all 13 statistics;
Layer-level, Increment, Update, and Late layer retain only the indicated
groups defined in Appendix~\ref{app:features}.}
\label{tab:source-group-ablation}
\small
\setlength{\tabcolsep}{5pt}
\begin{tabular}{lccccc}
\toprule
 & \textbf{Full} & \textbf{Layer-level} & \textbf{Increment}
 & \textbf{Update} & \textbf{Late layer} \\
\midrule
HaluEval2  & \best{83.4} & \second{83.3} & 82.3 & 65.3 & 82.0 \\
TrueFalse & \second{97.1} & \best{97.2} & 96.7 & 88.5 & 96.0 \\
HELM      & \best{88.2} & \second{87.7} & 86.0 & 62.4 & 86.8 \\
\bottomrule
\end{tabular}
\vspace{-5mm}
\end{table}

\subsection{Depth Subsampling Ablation}
\label{app:depth-weighting}

The random-intercept model (Remark~\ref{prop:ensemble}) associates depth
averaging with an innovation-variance term of order
$\sigma_\xi^2/m_{\mathrm{eff}}$, where $m_{\mathrm{eff}}$ is the support size
under the independent-innovation idealisation.  The ablation examines the
corresponding performance pattern through uniformly spaced supports and two
contiguous half-depth controls.

\paragraph{Subsampling Protocol.}
For $m_{\mathrm{eff}}\in\{2,4,8,16,m\}$, the uniformly spaced support is
\begin{equation}
  I_{m_{\mathrm{eff}}}^{\mathrm{unif}}
  :=
  \left\{
    \left\lfloor\frac{j(m-1)}{m_{\mathrm{eff}}-1}\right\rfloor
    :j=0,\ldots,m_{\mathrm{eff}}-1
  \right\}
  \subseteq\mathcal K_m.
\end{equation}
The two location controls use
$I^{\mathrm{first}}:=\{0,\ldots,\lfloor m/2\rfloor-1\}$ and
$I^{\mathrm{second}}:=\{\lfloor m/2\rfloor,\ldots,m-1\}$.  Each support mean
is fitted with the same scalar logistic calibration used by the depth-mean
detector and evaluated under five-fold cross-validation.  The figure displays
$m_{\mathrm{eff}}\in\{2,4,8,m\}$ together with the two half-depth controls.

\begin{figure}[H]
\centering
\includegraphics[width=0.95\linewidth]{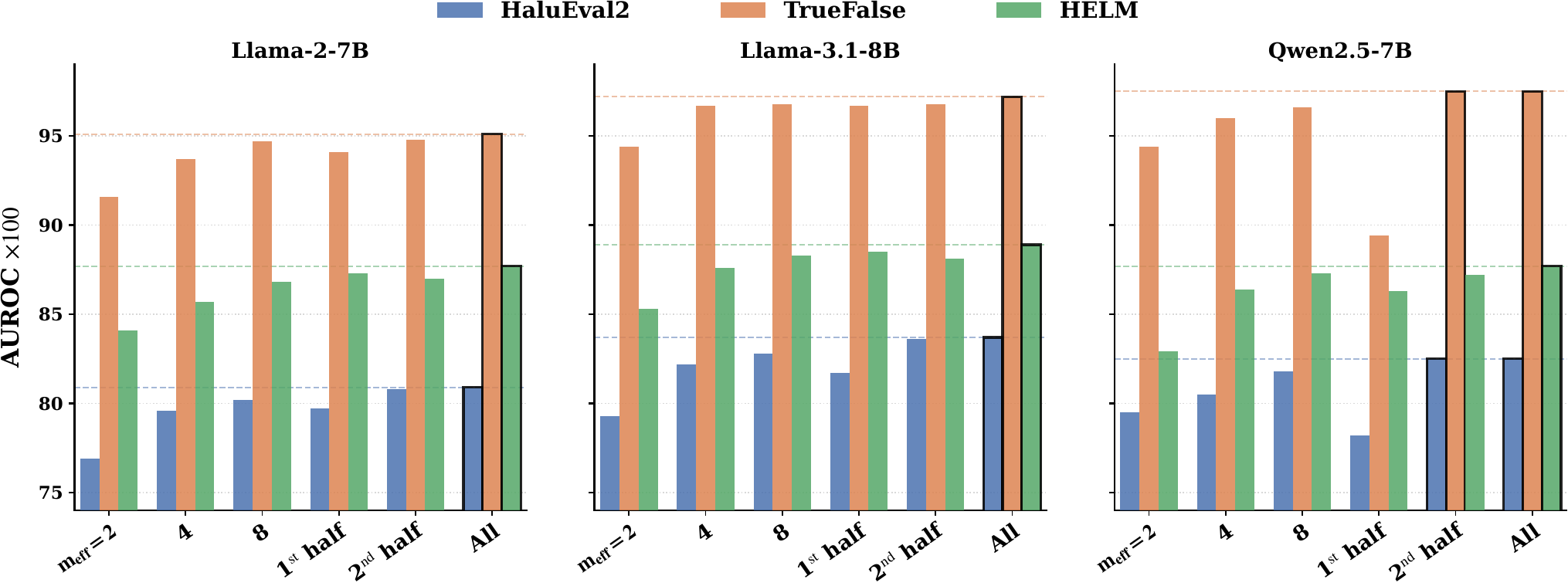}
\caption{\textbf{Depth-support ablation.} AUROC for uniformly spaced supports,
the first and second depth halves, and the full support across three models and
three datasets.}
\label{fig:depth-subsampling}
\vspace{-3mm}
\end{figure}

\paragraph{Key Observations.}
Figure~\ref{fig:depth-subsampling} reports the results across three models (Llama-2-7B, Llama-3.1-8B, Qwen2.5-7B) and three datasets (HaluEval2, TrueFalse, HELM).
Among the half-depth controls, the second half is stronger for Qwen2.5-7B,
while the two Llama models have smaller half-to-half gaps.

\paragraph{Consistency with the Variance Reduction Prediction.}
Under the independent-innovation idealisation, the innovation standard
deviation decreases as $\sigma_\xi/\sqrt{m_{\mathrm{eff}}}$.  The observed
diminishing change after eight uniformly spaced layers is compatible with that
directional prediction.

\section{Scope and Societal Considerations}
\label{app:discussion}

\subsection{Scope of Pre-Decoding Prediction}
\label{app:scope}

\textbf{Prediction unit and empirical scope.}
We evaluate truthfulness prediction from answer-onset states on the four
benchmarks.  Token-level error accumulation, attention
drift during long generation, and cross-turn context degradation call for
signals from the generated sequence or later interaction states.

\subsection{Length Control}
\label{app:probe-shortcut}



A classifier using statement length alone reaches 55.5 AUROC on HaluEval2,
52.3 on TrueFalse, and 52.8 on HELM, at least 25 points below the depth mean
of every evaluated model.  Lengths use a fixed tokenizer on the raw statement,
and the classifier is a balanced $L_2$ logistic fit with training-fold
standardisation.  For TruthfulQA, the score is computed before the first
answer token, so response length is unavailable to the detector by
construction.

\subsection{Societal Impacts}

\paragraph{Potential Benefits.}
The design offers two practical benefits.
\begin{enumerate}[nosep,leftmargin=1.5em]
  \item \textbf{Early risk scoring.}
  An answer-onset score can support gating, fallback, or additional checking
  before generation begins.  Its use is most direct in systems that already
  expose internal activations.

  \item \textbf{Low aggregation overhead.}
  After the probes at each layer are fitted, the depth mean adds linear time and
  constant streaming memory in the number of layers.  Probe construction and
  logistic calibration remain supervised offline steps.

\end{enumerate}

\paragraph{Potential Risks.}
The same scope and access assumptions create three deployment risks.
\begin{enumerate}[nosep,leftmargin=1.5em]
  \item \textbf{Error localisation.}
  An answer-onset score estimates risk for the response as a whole.  Locating
  and correcting individual factual errors requires checking the generated
  claims.

  \item \textbf{Distribution and shortcut sensitivity.}
  Training-fitted probes can change under topic, model, or dataset shift and
  can use correlates of the benchmark labels.  Deployment therefore requires
  calibration on the intended data distribution and monitoring after shifts.

  \item \textbf{Adaptive evasion.}
  White-box disclosure of internal scoring rules may enable prompts or model
  modifications that reduce the measured trajectory signature.  Adaptive
  attacks remain an unevaluated deployment condition.
\end{enumerate}

\noindent \thistool{} supplies an early risk signal within a broader
reliability pipeline.

\end{document}